\documentclass[11pt]{article}

\usepackage[bookmarks,colorlinks,breaklinks]{hyperref}  
\hypersetup{linkcolor=blue,citecolor=blue,filecolor=blue,urlcolor=blue} 
\usepackage{fullpage}
\usepackage{amsmath,amsfonts,amsthm,amssymb,xspace,bm, verbatim}
\usepackage{graphicx}
\usepackage{url,algorithm,algorithmic}

\newcommand{\newref}[2][]{\hyperref[#2]{#1~\ref*{#2}}}
\renewcommand{\eqref}[1]{\hyperref[#1]{(\ref*{#1})}}
  
\theoremstyle{plain}
\newtheorem{theorem}{Theorem}[section]
\newtheorem{lemma}[theorem]{Lemma}
\newtheorem{definition}[theorem]{Definition}
\theoremstyle{definition}

\newcommand{\iprod}[2]{\langle #1, #2 \rangle}   

\newcommand{\rn}{\mathbb{R}^n}

\newcommand{\rmn}{\mathbb{R}^{m \times n}}

\newcommand{\Tr}{\operatorname{Tr}}
\newcommand{\Var}{\operatorname{Var}}

\newcommand{\note}[1]{\marginpar{\tiny *note in TeX*}}
\newcommand{\ignore}[1]{}

\renewcommand{\phi}{\varphi}

\newcommand{\R}{\mathbb{R}}

\newcommand{\eqdef}{\stackrel{\textrm{def}}{=}}

\newcommand{\dist}{\mathrm{dist}}

\newcommand{\frob}[1]{\left\| {#1} \right\|_\text{F}}

\newcommand{\twonorm}[1]{\left\| {#1} \right\|_2}
\newcommand{\infnorm}[1]{\left\| {#1} \right\|_{\infty}}

\DeclareMathOperator*{\argmin}{argmin}

\newcommand{\calA}{\mathcal{A}}

\newcommand{\E}{\mathbb{E}}
\newcommand{\expec}[1]{\mathbb{E}\left[#1\right]}
\newcommand{\prob}[1]{\mathbb{P}\left[#1\right]}
\newcommand{\aff}{{\mathcal A}}
\newcommand{\Vt}{V^{t+1}}
\newcommand{\Vht}{\widehat{V}^{t+1}}
\newcommand{\Ut}{U^{t}}
\newcommand{\Uht}{\widehat{U}^{t}}
\newcommand{\Rt}{R^{(t+1)}}
\newcommand{\Vo}{V^*}
\newcommand{\So}{\Sigma^*}
\newcommand{\Uo}{U^*}
\newcommand{\Un}{U^N}
\newcommand{\Vn}{V^N}
\newcommand{\Sn}{\Sigma^N}
\newcommand{\sn}{\sigma^N}
\newcommand{\un}{u^N}
\newcommand{\so}{\sigma^*}
\newcommand{\vto}{v^{t+1}}
\newcommand{\ut}{u^t}
\newcommand{\utt}{(u^t)^\dag }
\newcommand{\uo}{u^* }
\newcommand{\uot}{(u^*)^\dag }
\newcommand{\vo}{v^* }

\newcommand{\Ait}{A_i^\dag }
\newcommand{\wvto}{\widehat{v}^{t+1}}

\newcommand{\ip}[2]{\langle #1, #2 \rangle}
\newcommand{\vt}{\vto}
\newcommand{\vht}{\wvto}

\newcommand{\as}{AltMinSense\ }
\newcommand{\ase}{AltMinSense}
\newcommand{\amc}{AltMinComplete\ }
\newcommand{\amce}{AltMinComplete}
\newcommand{\Vw}{\widehat{V}}
\newcommand{\Uw}{\widehat{U}}
\begin{document}

\title{ Low-rank Matrix Completion using Alternating Minimization}
%
%
%
%
%
%
\author{
%
%
Prateek Jain\\
{Microsoft Research India, Bangalore}\\
{prajain@microsoft.com}
\and
Praneeth Netrapalli\\
{The University of Texas at Austin}\\
{praneethn@utexas.edu}
\and
Sujay Sanghavi\\
{The University of Texas at Austin}\\
{sanghavi@mail.utexas.edu}
}

\maketitle
\begin{abstract}
Alternating minimization represents a widely applicable and empirically successful approach for finding low-rank matrices that best fit the given data. 
For example, for the problem of low-rank matrix completion, this method is believed to be one of the most accurate and efficient, and formed a major component of the winning entry in the Netflix Challenge \cite{Koren09,KorenBV09}. 

In the alternating minimization approach, the low-rank target matrix is written in a {\em bi-linear form}, i.e. $X = UV^\dag$; the algorithm then alternates between finding the best $U$ and the best $V$. Typically, each alternating step in isolation is convex and tractable. However the overall problem becomes non-convex and  there has been almost no theoretical understanding of when this approach yields a good result.

In this paper we present first theoretical analysis of  the performance of alternating minimization for matrix completion, and the related problem of matrix sensing. For both these problems, celebrated recent results have shown that they become well-posed and tractable once certain (now standard) conditions are imposed on the problem. We show that alternating minimization {\em also} succeeds under similar conditions. Moreover, compared to existing results, our paper shows that alternating minimization guarantees faster (in particular, geometric) convergence to the true matrix, while allowing a simpler analysis. 

%
%

\end{abstract}




\section{Introduction}
Finding a low-rank matrix to fit / approximate observations is a fundamental task in data analysis. In a slew of applications, a popular empirical approach has been to represent the target rank $k$ matrix $X\in \mathbb{R}^{m\times n}$  in a {\em bi-linear form} $X = UV^\dag$, where $U\in \mathbb{R}^{m\times k}$ and $V\in \mathbb{R}^{n\times k}$.  Typically, this is done for two  reasons: \\
{\em (a) Size and computation:} If the rank $k$ of the target matrix (to be estimated) is much smaller than $m,n$, then  $U,V$ are significantly smaller than  $X$ and hence are more efficient to optimize for. This is crucial for several practical applications, e.g., recommender systems where one routinely encounters matrices with billions of entries. \\
{\em (b) Modeling:} In several applications, one would like to impose extra constraints on the target matrix, besides just low rank. Oftentimes, these constraints might be easier and more natural to impose on factors $U$, $V$.  For example, in Sparse PCA \cite{ZouHT06}, one looks for a low-rank $X$ that is the product of {\em sparse} $U$ and $V$. 

Due to the above two reasons, in several applications, the target matrix $X$ is parameterized by $X=UV^\dag$. For example, clustering \cite{KimPark08b}, sparse PCA \cite{ZouHT06} etc. 

Using the bi-linear parametrization of the target matrix $X$, the task of estimating $X$ now reduces to finding  $U$ and $V$ that, for example, minimize an error metric. The resulting problem is typically non-convex due to bi-linearity. Correspondingly, a popular approach has been to use {\bf alternating minimization:} iteratively keep one of $U,V$ fixed and optimize over the other, then switch and repeat, see e.g. \cite{KorenBV09}. While the overall problem is non-convex, each sub-problem is typically convex and can be solved efficiently.

Despite wide usage of bi-linear representation and alternating minimization, there has been to date almost no theoretical understanding of when such a formulation works. Motivated by this disconnect between theory and practice in the estimation of  low-rank matrices, in this paper we provide the {\bf first guarantees for performance of alternating minimization}, for two low-rank matrix recovery problems: matrix completion, and matrix sensing. 

{\em Matrix completion} involves completing a low-rank matrix, by observing only a few of its elements. Its recent popularity, and primary motivation, comes from recommendation systems \cite{KorenBV09}, where the task is to complete a user-item ratings matrix using only a small number of ratings. As elaborated in Section~\ref{sec:results}, alternating minimization becomes particularly appealing for this problem as it provides a fast, distributed algorithm that can exploit both sparsity of ratings as well as the low-rank bi-linear parametrization of $X$. 

{\em Matrix sensing} refers to the problem of recovering a low-rank matrix $M\in \mathbb{R}^{m\times n}$ from affine equations. That is, given $d$ linear measurements $b_i = tr(A_i^\dag M)$ and measurement matrices $A_i$'s, the goal is to recover back $M$. This problem is particularly interesting in the case of $d\ll mn$ and was first studied in \cite{RechtFP2007} and subsequently in \cite{JainMD10, LeeB10}. In fact,  matrix completion is a special case of this problem, where each observed entry in the matrix completion problem  represents one single-element measurement matrix $A_i$. 

Without any extra conditions, both  matrix sensing and matrix completion are ill-posed problems, with potentially multiple low-rank solutions, and are in general NP hard \cite{MekaJCD2008}. Current work on these problems thus impose some extra conditions, which  makes the problems both well defined, and amenable to solution via the respective proposed algorithms \cite{RechtFP2007, CandesR2007}; see Section \ref{sec:related} for more details. In this paper, we show that {\bf under similar conditions to the ones used by the existing methods,} alternating minimization also guarantees recovery of the true matrix; we also show that it requires only a small number of computationally cheap iterations and hence, as observed empirically, is computationally much more efficient than the existing methods. \\
{\bf Notations}: We represent a matrix by capital letter (e.g. $M$) and a vector by small letter ($u$). $u_i$ represents $i$-th element of $u$ and $U_{ij}$ denotes $(i,j)$-th entry of $U$. $U_i$ represents $i$-th column of $U$ and $U^{(i)}$ represents $i$-th row of $U$. $A^\dag$ denotes matrix transpose of $A$. $u=vec(U)$ represents vectorized $U$, i.e., $u=[U_1^\dag\ U_2^\dag\ \dots\ U_k^\dag]^\dag$.  $\|u\|_p$ denotes $L_p$ norm of $u$, i.e., $\|u\|_p=(\sum_i|u_i|^p)^{1/p}$. By default, $\|u\|$ denotes $L_2$ norm of $u$. $\|A\|_F$ denotes Frobenius norm of $A$, i.e., $\|vec(A)\|_2$. $\|A\|_2=\max_{x, \|x\|_2=1} \|Ax\|_2$ denotes spectral norm of $A$. $tr(A)$ denotes the trace (sum of diagonal elements) of square matrix $A$. Typically, $\widehat{U}$, $\widehat{V}$ represent factor matrices (i.e., $\widehat{U}\in \mathbb{R}^{m\times k}$ and $\widehat{V}\in \mathbb{R}^{n\times k}$) and $U$, $V$ represent their orthonormal basis.
\section{Our Results}\label{sec:results}
In this section, we will  first define the matrix sensing problem, and present our results for it. Subsequently, we will do the same for matrix completion. The matrix sensing setting -- i.e. recovery of any low-rank matrix from linear measurements that satisfy matrix RIP -- represents an easier analytical setting than matrix completion, but still captures several  key properties of the problem that helps us in developing an analysis for matrix completion.
We note that for either problem, ours represent the first global optimality guarantees for alternating minimization based algorithms. 
\subsection*{Matrix Sensing via Alternating Minimization}
Given $d$ linear measurements $b_i = \langle M,A_i\rangle = tr(A_i^\dag M),\ 1\leq i\leq d$ of an {\em unknown} rank-$k$ matrix $M\in \mathbb{R}^{m\times n}$ and the sensing matrices $A_i, 1\leq i\leq d$, the goal in matrix sensing is to recover back $M$. 
In the following we  collate these coefficients, so that $b\in \mathbb{R}^{d}$ is the vector of $b_i$'s, and $\aff(\cdot):\mathbb{R}^{m\times n}\rightarrow d$ is the corresponding linear map, with $b=\aff(M)$. With this notation, the Low-Rank Matrix Sensing problem is:
\renewcommand{\theequation}{LRMS}
\begin{equation}\label{eq:lrms}
 \hspace*{-10pt}\text{Find }X\in \rmn, \;\; s.t \;\; {\mathcal A}(X) = b,\ \  rank(X)\leq k.
\end{equation}
\renewcommand{\theequation}{\arabic{equation}}
\addtocounter{equation}{-1}
As in the existing work \cite{RechtFP2007} on this problem, we are interested in the  {\bf under-determined case}, where $d < mn$. Note that this problem is a strict generalization of the popular compressed sensing problem \cite{CandesT2005}; compressed sensing represents the case when $M$ is restricted to be a diagonal matrix.

For matrix sensing, alternating minimization approach involves representing $X$ as a product of two matrices $U\in \mathbb{R}^{m\times k}$ and $V\in \mathbb{R}^{n\times k}$, i.e., $X=UV^\dag$. If $k$ is (much) smaller than $m,n$, these matrices will be (much) smaller than $X$. With this bi-linear representation, alternating minimization can be viewed as an approximate way to solve the following non-convex optimization problem: 
\[
\min_{U\in \mathbb{R}^{m\times k}, V\in \mathbb{R}^{n\times k}}  \quad  \|\aff(UV^\dag) - b\|^2_2 
\]
As mentioned earlier, alternating minimization algorithm for matrix sensing  now alternately solves for $U$ and $V$ while fixing the other factor. See Algorithm~\ref{algo:sensing-altmin} for a pseudo-code of  \as algorithm that we analyze. 

We note  two key properties of \as: a) Each minimization -- over $U$ with $V$ fixed, and vice versa -- is a simple least-squares problem, which can be solved in time $O(dn^2k^2+n^3k^3)$\footnote{Throughout this paper, we assume $m\leq n$.}, b) We initialize $U^0$ to be the top-$k$ left singular vectors of  $\sum_i A_i b_i$ (step 2 of Algorithm~\ref{algo:sensing-altmin}). As we will see later in Section~\ref{sec:sensing}, this provides a good initialization point for the sensing problem which is crucial; if the first iterate $\widehat{U}^0$ is orthogonal, or almost orthogonal, to the true $U^*$ subspace, \as  may never converge to the true space (this is easy to see in the simplest case, when the map is identity, i.e. $\aff(X) = X$ -- in which case \as just becomes the power method). 


\begin{algorithm}[th]
\caption{{\bf \as:} Alternating minimization for matrix sensing}
\label{algo:sensing-altmin}
\begin{algorithmic}[1]
\STATE Input $b,\aff$
\STATE Initialize $\widehat{U}^0$ to be the top-$k$ left singular vectors of $\sum_i A_i b_i$
\FOR{$t=0,\cdots,T-1$}
	\STATE $\widehat{V}^{t+1}\leftarrow \argmin_{V\in \R^{n\times k}} \  \|\aff(\widehat{U}^t \,V^\dag )-b\|_2^2$
	\STATE $\widehat{U}^{t+1}\leftarrow \argmin_{U\in \mathbb{R}^{m\times k}} \  \|\aff(U\,(\widehat{V}^{t+1})^\dag )-b\|_2^2$
\ENDFOR
\STATE Return $X=\widehat{U}^T(\widehat{V}^T)^\dag $
\end{algorithmic}
\end{algorithm}


In general, since $d<mn$, problem \eqref{eq:lrms} is not well posed as there can be multiple rank-$k$ solutions that satisfy $\calA(X)=b$. However, inspired by a similar condition in compressed sensing \cite{CandesT2005}
, Recht et al. \cite{RechtFP2007} showed that if the linear map $\aff$ satisfies a {\em (matrix) restricted isometry property} (RIP), then a trace-norm based convex relaxation of \eqref{eq:lrms} leads to exact recovery. This property is defined below. 
\begin{definition}\cite{RechtFP2007}\label{defn:rip}
A linear operator $\aff(\cdot):\rmn\rightarrow \mathbb{R}^d$ is said to satisfy $k$-RIP, with $\delta_{k}$ RIP constant, if for all $X\in \rmn$ s.t. rank$(X)\leq k$, the following holds:
\begin{align}\label{eqn:sensing-2k-RIP1}
(1-\delta_k) \left\|X\right\|_F^2 \leq \left\|\calA(X)\right\|_2^2 \leq (1+\delta_k) \left\|X\right\|_F^2.
\end{align}
\end{definition}
Several random matrix ensembles with sufficiently many measurements ($d$) satisfy  matrix RIP \cite{RechtFP2007}. For example, if $d=\Omega(\frac{1}{\delta_k^2}kn\log n)$ and each entry of $A_i$ is sampled i.i.d. from a $0$-mean sub-Gaussian distribution then $k$-RIP is satisfied with RIP constant  $\delta_k$. 

We now present our main result for \ase.
\begin{theorem}\label{thm:sensing_err}
 Let $M=U^*\Sigma^*V^{*^\dag }$ be a rank-$k$ matrix with non zero singular values $\sigma_1^*\geq \sigma_2^*\dots\geq \sigma_k^*$. Also, let the linear measurement operator $\calA(\cdot):\rmn \rightarrow \mathbb{R}^d$ satisfy $2k$-RIP with RIP constant $\delta_{2k}<\frac{(\so_k)^2}{(\so_1)^2}\,\frac{1}{100k}$.  Then, in the \as algorithm (Algorithm \ref{algo:sensing-altmin}), for all $T> 2\log (\|M\|_F/\epsilon)$,  the iterates $\widehat{U}^{T}$ and $\widehat{V}^T$ satisfy:
 $$\|M-\widehat{U}^{T}(\widehat{V}^T)^\dag\|_F\leq \epsilon.$$
\end{theorem}
The above theorem establishes geometric convergence (in $O(\log(1/\epsilon))$ steps) of \as to the optimal solution of \eqref{eq:lrms} under standard RIP assumptions. 
This is in contrast to existing iterative methods for trace-norm minimization all of which require at least $O(\frac{1}{\sqrt{\epsilon}})$ steps; interior point methods for trace-norm minimization converge to the optimum in $O(\log(1/\epsilon))$ steps but require storage of the full $m\times n$ matrix and require $O(n^5)$ time per step, which makes it infeasible for even moderate sized problems.

Recently, several projected gradient based methods have been developed for matrix sensing \cite{JainMD10, LeeB10} that also guarantee convergence to the optimum in $O(\log (1/\epsilon))$ steps. But each iteration in these algorithms requires computation of the top $k$ singular components of an $m\times n$ matrix, which is typically significantly slower  than solving a least squares problem (as required by each iteration of \ase).\\
{\bf Stagewise AltMinSense Algorithm}: A drawback of our analysis for \as is the dependence of $\delta_{2k}$ on the condition number ($\kappa=\frac{\so_1}{\so_k}$) of $M$, which implies that the number of measurements $d$ required by \as grows quadratically with $\kappa$. We address this issue by using a stagewise version of \as (Algorithm~\ref{algo:sensing-altmin-improved}) for which we are able to obtain near optimal measurement requirement. 

The key idea behind our stagewise algorithm is that if one of the singular vectors of $M$ is very dominant, then we can treat the underlying matrix as a rank-$1$ matrix plus noise and approximately recover the top singular vector. Once we remove this singular vector from the measurements, we will have a relatively well-conditioned problem. Hence, at each stage of Algorithm~\ref{algo:sensing-altmin-improved}, we seek to remove the remaining most dominant singular vector of $M$. See Section~\ref{sec:sam} for more details; here we state the corresponding theorem regarding the performance of Stage-AltMin. 
\begin{theorem}\label{thm:sam_err}
Let $M=U^*\Sigma^*V^{*^\dag }$ be a rank-$k$ incoherent matrix with non zero singular values $\sigma_1^*\geq \sigma_2^*\dots\geq \sigma_k^*$. Also, let $\calA(\cdot):\rmn \rightarrow \mathbb{R}^d$ be a linear measurement operator that satisfies $2k$-RIP with RIP constant $\delta_{2k}<\frac{1}{3200k^2}$.  Suppose, Stage-AltMin (Algorithm \ref{algo:sensing-altmin-improved}) is supplied inputs $\aff$, $b=\aff(M)$. Then, the $i$-th stage iterates $\Uw^T_{1:i}$, $V^T_{1:i}$ satisfy: $$\|M-\Uw^T_{1:i}\left(V^T_{1:i}\right)^{\dag}\|_F^2\leq \max(\epsilon, 16k(\so_{i+1})^2),$$
where $T= \Omega\left(\log(\|M\|_F^2/\epsilon)\right)$. That is, the $T$-th step iterates of the $k$-th stage, satisfy: $\|M-\Uw^T_{1:k}\left(V^T_{1:k}\right)^{\dag}\|_F^2\leq \epsilon.$
\end{theorem}
The above theorem guarantees exact recovery using $O(k^4 n \log n)$ measurements which is only $O(k^3)$ worse than the information theoretic lower bound. We also note that for simplicity of analysis, we did not optimize the constant factors in $\delta_{2k}$.
\subsection*{Matrix Completion via Alternating Minimization}
The matrix completion problem is the following: there is an unknown rank-$k$ matrix $M\in \mathbb{R}^{m\times n}$, of which we know a set $\Omega \subset [m]\times [n]$ of elements; that is, we know the values of elements $M_{ij}$, for $(i,j)\in\Omega$. The task is to recover $M$. Formally, the Low-Rank Matrix Completion problem is:\renewcommand{\theequation}{LRMC}
\begin{equation}
\mbox{Find rank-}k \mbox{ matrix } X \mbox{ s.t. }  P_\Omega(X)=P_\Omega(M),\label{eq:matcomp}
 \end{equation} \renewcommand{\theequation}{\arabic{equation}} \addtocounter{equation}{-1}
where for any matrix $S$ and a set of elements $\Omega \subset [m]\times [n]$ the matrix $P_\Omega(S)\in \mathbb{R}^{m\times n}$ is as defined below:
\begin{equation}
  P_\Omega(S)_{ij}=\begin{cases}S_{ij}&\mbox{ if }(i,j)\in \Omega,
\\ 0&\mbox{ otherwise.}
\end{cases}\label{eq:pomega}
\end{equation}
We are again interested in the under-determined case; in fact, for a fixed rank $k$, as few as $O(n\log n)$ elements may be observed. This problem is a special case of matrix sensing, with the measurement matrices $A_i=e_je_\ell^\dag$ being non-zero only in single elements; however, such matrices do not satisfy  matrix RIP conditions like (\ref{eqn:sensing-2k-RIP1}). For example, consider a low-rank $M=e_1e_1^\dag$ for which a uniformly random $\Omega$ of size $O(n\log n)$ will most likely miss the non-zero entry of $M$.

Nevertheless, like matrix sensing, matrix completion has been shown to be possible once additional conditions are applied to the low-rank matrix $M$ and the observation set $\Omega$. Starting with the first work \cite{CandesR2007}, the typical assumption has been to have $\Omega$ generated uniformly at random, and $M$ to satisfy a particular incoherence property that, loosely speaking, makes it very far from a sparse matrix. In this paper, we show that {\em once} such assumptions are made, alternating minimization {\em also} succeeds. We now restate, and subsequently use, this incoherence definition.

\begin{definition}\cite{CandesR2007}\label{defn:sampling-incoherentmatrices}
A matrix $M\in \rmn$ is incoherent with parameter $\mu$ if:
\begin{align}\label{eqn:sampling-incoherentmatrices}
  \twonorm{u^{(i)}} \leq \frac{\mu \sqrt{k}}{\sqrt{m}} \;\forall\;i \in [m],\  \twonorm{v^{(j)}} \leq \frac{\mu \sqrt{k}}{\sqrt{n}} \;\forall\;j \in [n],
\end{align}
where $M=U\Sigma V^T$ is the SVD of $M$ and $u^{(i)}$, $v^{(j)}$ denote the $i^{\textrm{th}}$ row of $U$
and the $j^{\textrm{th}}$ row of $V$ respectively.
\end{definition}
The alternating minimization algorithm can be viewed as an approximate way to solve the following non-convex problem:
\[
\min_{U,V\in \mathbb{R}^{n\times k}}  \quad  \| P_\Omega(UV^\dag) - P_\Omega(M)\|^2_F 
\]
Similar to \ase, the altmin procedure proceeds by alternatively solving for $U$ and $V$.  As noted earlier, this approach has been popular in practice and has seen several variants and extensions being used in practice \cite{Zhouetal08,KorenBV09,Koren09,chen2012matrix}. However, for ease of analysis, our algorithm further modifies the standard alternating minimization method. In particular, we introduce  partitioning of the observed set $\Omega$, so that we use different partitions of $\Omega$ in each iteration. See Algorithm~\ref{algo:completion-altmin} for a pseudo-code of our variant of the alternating minimization approach.
\begin{algorithm}[th]
\caption{{\bf AltMinComplete:} Alternating minimization for matrix completion}
\label{algo:completion-altmin}
\begin{algorithmic}[1]
\STATE Input: observed set $\Omega$, values $P_{\Omega}(M)$
\STATE Partition $\Omega$ into $2T+1$ subsets $\Omega_0,\cdots,\Omega_{2T}$ with each element of $\Omega$ belonging to one of the $\Omega_t$ with equal probability (sampling with replacement)
\STATE $\widehat{U}^{0}=SVD(\frac{1}{p}P_{\Omega_0}(M), k)$ i.e., top-$k$ left singular vectors of $\frac{1}{p}P_{\Omega_0}(M)$
\STATE Clipping step : Set all elements of $\widehat{U}^{0}$ that have magnitude greater than $\frac{2\mu \sqrt{k}}{\sqrt{n}}$ to zero and orthonormalize the columns of $\widehat{U}^0$
\FOR{$t=0,\cdots,T-1$}
	\STATE $\widehat{V}^{t+1}\leftarrow \argmin_{V\in \R^{n\times k}} \  \|P_{\Omega_{t+1}}(\widehat{U}^t V^\dag - M)\|_F^2$
	\STATE $\widehat{U}^{t+1}\leftarrow \argmin_{U\in \mathbb{R}^{m\times k}} \  \|P_{\Omega_{T+t+1}}(U\left(\widehat{V}^{t+1}\right)^\dag - M) \|_F^2$
\ENDFOR
\STATE Return $X=\widehat{U}^T(\widehat{V}^T)^\dag $
\end{algorithmic}
\end{algorithm}

We now present our main result for \eqref{eq:matcomp}:
\begin{theorem}\label{thm:sampling-altmin-main}
Let $M=U^*\Sigma^*V^{*^\dag } \in \mathbb{R}^{m\times n}$ ($n\geq m$) be a rank-$k$ incoherent matrix, i.e., both $U^*$ and $V^*$ are $\mu$-incoherent (see Definition~\ref{defn:sampling-incoherentmatrices}).  Also, let each entry of $M$ be observed uniformly and independently with probability, 
$$p>C\frac{\left(\frac{\sigma_1^*}{\sigma_k^*}\right)^2\mu^2k^{2.5}\log n \log \frac{k\|M\|_F}{\epsilon}}{m\delta_{2k}^2},$$
where $\delta_{2k}\leq \frac{\so_k}{12k\so_1}$ and $C>0$ is a global constant. 
 Then w.h.p. for $T = C'\log\frac{\frob{M}}{\epsilon}$, the outputs $\widehat{U}^{T}$ and $V^{T}$ of  Algorithm \ref{algo:completion-altmin}, with input $(\Omega, P_\Omega(M))$ (see Equation \eqref{eq:pomega}) satisfy: 
$  \frob{M - \widehat{U}^{T}\left(V^{T}\right)^{\dag}} \leq \epsilon.$
\end{theorem}
The above theorem implies that by observing $\left|\Omega\right|=O\left((\frac{\so_1}{\so_k})^4k^{4.5}n \log n\log(k\|M\|_F/\epsilon)\right)$ random entries of an incoherent $M$, \amc can recover $M$ in $O(\log(1/\epsilon))$ steps. In terms of sample complexity ($|\Omega|$), our results show alternating minimization may require a bigger $\Omega$ than convex optimization, as our result has $|\Omega|$ depend on the condition number, required accuracy ($\epsilon$) and worse dependence on $k$ than known  bounds. In contrast, trace-norm minimization based methods require $O(kn \log n)$ samples only. 

Empirically however, this is not seen to be the case \cite{JainMD10} and we leave further tightening of the sample complexity bounds for matrix completion as an open problem. 

In terms of time complexity, we show that \amc needs  time $O(|\Omega| k^2 \log (1/\epsilon))$. This is in contrast to popular trace-norm minimization based methods that need $O(1/\sqrt{\epsilon})$ steps \cite{CaiCS2008} and total time complexity of $O(|\Omega|n/\sqrt{\epsilon})$; note that the latter can be potentially quadratic in  $n$. Furthermore, each step of such methods requires computation of the SVD of an $m\times n$ matrix. As mentioned earlier, interior point methods for trace-norm minimization also converge in $O(\log (1/\epsilon))$ steps but each iteration requires $O(n^5)$ steps and need storage of the entire $m\times n$ matrix $X$.

\section{Related Work}\label{sec:related}
{\bf Alternating Minimization}: Alternating minimization and its variants have been applied to several low-rank matrix estimation problems. For example, clustering \cite{KimPark08b}, sparse PCA \cite{ZouHT06}, non-negative matrix factorization \cite{KimPar08}, signed network prediction \cite{HsiehCD12} etc. There are three main reasons for such wide applicability of this approach: a) low-memory footprint and fast iterations, b) flexible modeling, c) amenable to parallelization. However, despite such empirical success, this approach has largely been used as a heuristic and has had no theoretical analysis other than the guarantees of convergence to the {\em local minima} \cite{Zangwill69}. Ours is the first analysis of this approach for two practically important problems: a) matrix completion, b) matrix sensing. 

\noindent {\bf Matrix Completion:} This is the problem of completing a low-rank matrix from a few sampled entries. Candes and Recht  \cite{CandesR2007} provided the first results on this problem, showing that under the random sampling and incoherence conditions (detailed above), $O(kn^{1.2}\log n)$ samples allow for recovery via convex trace-norm minimization; this was improved to $O(kn\log n)$ in \cite{CandesT2009}. For large matrices, this approach is not very attractive due to the need to store and update the entire matrix, and because iterative methods for trace norm minimization require  $O(\frac{1}{\sqrt{\epsilon}})$ steps to achieve additive error of $\epsilon$. Moreover, each such step needs to compute an SVD. 

Another  approach, in \cite{KeshavanOM2009}, involved taking a single SVD, followed by gradient descent  on a Grassmanian manifold. However,  {\em (a)} this is more expensive than alternating minimization as it needs to compute gradient over Grassmanian manifold which in general is a computationally intensive step, and {\em (b)} the analysis of the algorithm only guarantees asymptotic convergence, and in the worst case might take exponential time in the problem size. 

Recently, several other matrix completion type of problems have been studied in the literature. For example, robust PCA \cite{ChandrasekaranSPW11,CandesLMW11}, spectral clustering \cite{JalaliCSX11} etc. Here again, under additional assumptions, convex relaxation based methods have rigorous analysis but alternating minimization based algorithms continue to be  algorithms of choice in practice. 
 
\noindent {\bf Matrix Sensing:} The general problem of matrix sensing was first proposed by \cite{RechtFP2007}. They  established recovery via trace norm minimization, assuming the sensing operator satisfies ``restricted isometry" conditions. Subsequently, several other methods \cite{JainMD10, LeeB10}  were proposed for this problem that also recovers the underlying matrix with optimal number of measurements and can give an $\epsilon$-additive approximation in time $O(\log(1/\epsilon)$. But, similar to matrix completion, most of  these methods require computing SVD of a large matrix at each step and hence have poor scalability to large problems.  

We show that \as and AltMin-Completion provide more scalable algorithms for their respective problems. We  demonstrate that these algorithms have geometric convergence to the optima, while each iteration is relatively cheap. For this, we assume conditions similar to those required by existing algorithms; albeit, with one drawback: number of samples required by our analysis depend on the condition number of the underlying matrix $M$. For the matrix sensing problem, we remove this requirement by using a stagewise algorithm; we leave similar analysis for matrix completion as an open problem.

\section{Matrix Sensing}\label{sec:sensing}
In this section, we  study the matrix sensing problem \eqref{eq:lrms} and prove that if the measurement operator, $\aff$, satisfies RIP then \as (Algorithm \ref{algo:sensing-altmin}) recovers the underlying low-rank matrix {\em exactly} (see Theorem~\ref{thm:sensing_err}).

At a high level, we prove Theorem \ref{thm:sensing_err} by showing that  the ``distance'' between subspaces spanned by $\widehat{V}^t$ (iterate at time $t$) and $V^*$ decreases exponentially with $t$. This done based on the observation that once the (standard) matrix RIP condition (Definition \ref{defn:rip}) holds, alternating minimization can be viewed, and analyzed, as a {\bf perturbed version of the power method.} This is easiest to see for the rank-1 case below; we detail this proof, and then the more general rank-$k$ case.

In this paper, we use the following definition of distance between subspaces:
\begin{definition}\cite{GolVan96}\label{defn:dist}
  Given two matrices $\widehat{U},\widehat{W}\in \mathbb{R}^{m\times k}$, the (principal angle) distance between the subspaces spanned by the columns of $\widehat{U}$ and $\widehat{W}$ is given by:
\begin{align*}
  \dist\left(\widehat{U},\widehat{W}\right) \eqdef \twonorm{U_{\perp}^\dag W} = \twonorm{W_{\perp}^\dag U}
\end{align*}
where $U$ and $W$ are orthonormal bases of the spaces $\textrm{Span}\left(\widehat{U}\right)$ and $\textrm{Span}\left(\widehat{W}\right)$, respectively. Similarly, $U_{\perp}$ and $W_{\perp}$ are any orthonormal bases of the perpendicular spaces $\textrm{Span}\left(U\right)^{\perp}$ and $\textrm{Span}\left(W\right)^{\perp}$, respectively.
\end{definition}
{\bf Note:} {\em (a)} The distance depends only on the spaces spanned by the columns of $\widehat{U},\widehat{W}$, {\em (b)}
if the ranks of $\widehat{U}$ and $\widehat{W}$ (i.e. the dimensions of their spans) are not equal, then $\dist\left(\widehat{U},\widehat{W}\right) = 1$, and {\em (c)} $\dist\left(\widehat{U},\widehat{W}\right) = 0$ if and only if they span the same subspace of $\mathbb{R}^m$.

We now present a theorem that bounds the distance between the subspaces spanned by $\widehat{V}^t$ and $V^*$ and show that it decreases exponentially with $t$.
\begin{theorem}\label{thm:sensing}
Let $b=\aff(M)$ where  $M$ and $\aff$ satisfy assumptions given in Theorem~\ref{thm:sensing_err}. Then, the $(t+1)$-th iterates $\widehat{U}^{t+1}$, $\widehat{V}^{t+1}$ of \as satisfy:
\begin{align*}
  \dist\left(\widehat{V}^{t+1},V^*\right) &\leq \frac{1}{4}\cdot\dist\left(\widehat{U}^{t}, U^*\right)  \mbox{ , } \\
  \dist\left(\widehat{U}^{t+1},U^*\right) &\leq \frac{1}{4}\cdot\dist\left(\widehat{V}^{t+1}, V^*\right)
\end{align*}
where 
$\dist\left(U, W\right)$ denotes the principal angle based distance (see Definition~\ref{defn:dist}).
\end{theorem}
Using Theorem \ref{thm:sensing}, we are now ready to prove Theorem \ref{thm:sensing_err}.
\begin{proof}[Proof Of Theorem~\ref{thm:sensing_err}]
Assuming correctness of Theorem~\ref{thm:sensing}, Theorem~\ref{thm:sensing_err} follows by using the following set of inequalities:
\begin{align*}
  \|M-\widehat{U}^T(\widehat{V}^T)^\dag\|_F^2&\stackrel{\zeta_1}{\leq} \frac{1}{1-\delta_{2k}}\|\aff(M-\widehat{U}^T(\widehat{V}^T)^\dag)\|_2^2,\\
&\stackrel{\zeta_2}{\leq} \frac{1}{1-\delta_{2k}}\|\aff(M(I-V^T(V^T)^\dag))\|_2^2,\\
&\stackrel{\zeta_3}{\leq} \frac{1+\delta_{2k}}{1-\delta_{2k}}\|U^*\Sigma^*(V^*)^\dag(I-V^T(V^T)^\dag))\|_F^2,\\
&\stackrel{\zeta_4}{\leq} \frac{1+\delta_{2k}}{1-\delta_{2k}} \|M\|_F^2\dist^2\left(V^T, V^*\right)\stackrel{\zeta_5}{\leq} \epsilon,
\end{align*}
where $V^T$ is an orthonormal basis of $\widehat{V}^T$, $\zeta_1$ and $\zeta_3$ follow by RIP, $\zeta_2$ holds as $\widehat{U}^T$ is the least squares solution,
$\zeta_4$ follows from the definition of $\dist(\cdot,\cdot)$ and finally $\zeta_5$ follows from Theorem~\ref{thm:sensing} and by setting $T$ appropriately.
\end{proof}
To complete the proof of Theorem~\ref{thm:sensing_err}, we now need to prove Theorem~\ref{thm:sensing}. In the next section, we illustrate the main ideas of the proof of Theorem~\ref{thm:sensing} by applying it to a rank-$1$ matrix i.e., when $k = 1$. We then provide a proof of Theorem \ref{thm:sensing} for arbitrary $k$ in Section~\ref{sec:sensing_k}.
\subsection{Rank-$1$ Case}\label{sec:sensing_1}
In this section, we provide a proof of Theorem~\ref{thm:sensing} for the special case of $k=1$. That is, let $M=u^*\sigma^*(v^*)^\dag $ s.t. $u^*\in\mathbb{R}^{m},\ \|u^*\|_2=1$ and $v^*\in\mathbb{R}^{n}, \|v^*\|_2=1$. Also note that when $\widehat{u}$ and $\widehat{w}$ are vectors, $\dist(\widehat{u}, \widehat{w})=1-(u^\dag w)^2$, where $u=\widehat{u}/\|\widehat{u}\|_2$ and $w=\widehat{w}/\|\widehat{w}\|_2$. 

Consider the $t$-th update step in the \as procedure. As $\widehat{v}^{t+1}=$\\
$\argmin_{\widehat{v}}\sum_{i=1}^d\left(\widehat{u}^{t\dag} A_i^\dag \widehat{v}-\so u^{*^\dag} A_i^\dag v^*\right)^2$,
setting the gradient of the above objective function to $0$, we obtain:
\begin{align*}
  \left(\sum_{i=1}^d A_iu^t(u^t)^\dag A_i^\dag \right)\|\widehat{u}^t\|_2\wvto&=\sigma^*\left(\sum_{i=1}^d A_iu^tu^{*^\dag} A_i^\dag \right)\vo,
\end{align*}
where $u^t=\widehat{u}^t/\|\widehat{u}^t\|_2$. 
Now, let $B=\sum_{i=1}^d A_iu^t(u^t)^\dag A_i^\dag$ and $C=\sum_{i=1}^d A_iu^t(u^*)^\dag A_i^\dag$. Then, 
\begin{align}
  \|\widehat{u}^t\|_2\wvto&=\sigma^*B^{-1}C\vo,\nonumber\\
&=\underbrace{\ip{\uo}{\ut}\sigma^*\vo}_{\text{Power Method}}-\underbrace{B^{-1}\left(\ip{\uo}{\ut}B-C\right)\sigma^*\vo}_{\text{Error Term}}. \label{eq:error_1}
\end{align}
Note that the first term in the above expression is the power method iterate (i.e., $M^{\dag} u^t$). The second term is an error term and the goal is to show that it becomes smaller as $\ut$ gets closer to $\uo$. Note that when $\ut=\uo$, the error term is $0$ {\em irrespective} of the measurement operator $\aff$.

Below, we provide a precise bound on the error term:
\begin{lemma}
\label{lem:error_1}
Consider the error term defined in \eqref{eq:error_1} and let $\aff$ satisfy $2$-RIP with constant $\delta_{2}$. Then, 
$$\|B^{-1}\left(\ip{\uo}{\ut}B-C\right)\vo\|\leq \frac{3\delta_{2}}{1-3\delta_{2}}\sqrt{1-\ip{\ut}{\uo}^2}$$
\end{lemma}
See Appendix~\ref{app:sensing1} for a detailed proof of the above lemma. 

Using the above lemma, we now finish the proof of Theorem~\ref{thm:sensing}: 
\begin{proof}[Proof of Rank-$1$ case of Theorem~\ref{thm:sensing}]
Let $\vt=\wvto/\|\wvto\|_2$. Now, using \eqref{eq:error_1} and Lemma~\ref{lem:error_1}:, 
\begin{align*}
  &\ip{\vto}{\vo}=\frac{\ip{\wvto}{\vo}}{\|\wvto\|}=\frac{\ip{\wvto/\sigma^*}{\vo}}{\|\wvto/\sigma^*\|} \nonumber\\
&\leq \frac{\ip{\uo}{\ut}-\widehat{\delta}_2\sqrt{1-\ip{\uo}{\ut}^2}}{\sqrt{\left(\ip{\uo}{\ut}-\widehat{\delta}_2\sqrt{1-\ip{\uo}{\ut}^2}\right)^2+\widehat{\delta}_2^2\left(1-\ip{\uo}{\ut}^2\right)}},
\end{align*}
where $\widehat{\delta}_2= \frac{3\delta_2}{1-3\delta_2}$. That is, 
\begin{align*}&\dist^2(\vto, \vo) \leq \frac{\widehat{\delta}_2^2(1-\ip{\uo}{\ut}^2)}{(\ip{\uo}{\ut}-\widehat{\delta}_2\sqrt{1-\ip{\uo}{\ut}^2})^2+\widehat{\delta}_2^2(1-\ip{\uo}{\ut}^2)},\end{align*}
Hence, {\em assuming} $\ip{\uo}{\ut}\geq 5\widehat{\delta}_2$, $\dist(\vto, \vo)\leq \frac{1}{4}\dist(\ut, \uo).$
As $\dist(u^{t+1}, \uo)$ and $\dist(v^{t+1}, \vo)$ are decreasing with $t$ (from the above bound), we only need to show that $\ip{u^0}{\ut}\geq 5\widehat{\delta}_2$. 
Recall that $\widehat{u}^0$ is obtained by using one step of SVP algorithm \cite{JainMD10}. Hence, using Lemma 2.1 of \cite{JainMD10} (see Lemma~\ref{lem:jmd}): 
$$\|\sigma_1^*(I-u^0(u^0)^\dag)\uo)\|_2^2\leq \|M-\widehat{u}^0(\widehat{v}^0)^\dag\|_F^2\leq 2\delta_{2}\|M\|_F^2.$$
Therefore, $\ip{u^0}{\uo}\geq \sqrt{1-2\delta_2}\geq 5\widehat{\delta}_2$ assuming $\delta_{2}\leq \frac{1}{100}$. 
\end{proof}
\subsection{Rank-$k$ Case}
\label{sec:sensing_k}
In this section, we present the proof of Theorem \ref{thm:sensing} for arbitrary $k$, i.e., when $M$ is a rank-$k$ matrix (with SVD $U^*\Sigma^* \left(V^*\right)^{\dag})$.

Similar to the analysis for the rank-$1$ case (Section \ref{sec:sensing_1}), we show that even for arbitrary $k$, the updates of AltMinSense are essentially power-method type updates but with a bounded error term whose magnitude decreases with each iteration.

However, directly analyzing iterates of \as is a bit tedious due to non-orthonormality of intermediate iterates $\widehat{U}$. Instead, {\bf for analysis only} we consider the iterates of a modified version of \ase, where we explicitly orthonormalize each iterate using the QR-decomposition\footnote{The QR decomposition factorizes a matrix into an orthonormal matrix (a basis of its column space) and an upper triangular matrix; that is given $\widehat{S}$ it computes $\widehat{S} = SR$ where $S$ has orthonormal columns and $R$ is upper triangular. If $\widehat{S}$ is full-rank, so are $S$ and $R$.}. In particular, suppose we replace steps 4 and 5 of \ase with the following 
\begin{align}
\widehat{U}^{t}&= \, U^{t}R_U^{t}\ \ \text{ (QR decomposition), }\nonumber\\
\widehat{V}^{t+1}&\leftarrow \, \argmin_{V} \  \|\aff(U^t V^\dag )-b\|_2^2,\nonumber\\ 
\widehat{V}^{t+1}&= \, V^{t+1}R_V^{t+1}\ \ \text{ (QR decomposition) } \nonumber  \\
\widehat{U}^{t+1}&\leftarrow \, \argmin_{U} \  \|\aff(U (V^{t+1})^\dag )-b\|_2^2 \label{eq:update_k1}
\end{align}
In our algorithm, in each iterate both $\widehat{U}^{t},\widehat{V}^{t}$ remain full-rank because $\dist\left(U^t,U^*\right)<1$; with this, the following lemma implies that the spaces spanned by the iterates in our \ase ~ algorithm are {\em exactly the same} as the respective ones by the iterates of the above modified version (and hence the distances $\dist(\widehat{U}^{t},U^*)$ and $\dist(\widehat{V}^{t},V^*)$ are also the same for the two algorithms).

\begin{lemma}\label{lemma:QR-irrelevance}
Let $\widehat{U}^t$ be the $t^{th}$ iterate of our AltMinSense algorithm, and $\widetilde{U}^t$ of the modified version stated above. Suppose also that both $\widehat{U}^t,\widetilde{U}^t$ are full-rank, and span the same subspace. Then the same will be true for the subsequent iterates for the two algorithms, i.e. $Span(\widehat{V}^{t+1}) = Span(\widetilde{V}^{t+1})$, $Span(\widehat{U}^{t+1}) = Span(\widetilde{U}^{t+1})$, and all matrices at iterate $t+1$ will be full-rank.
\end{lemma}

The proof of the above lemma can be found in Appendix \ref{app:sensingk}. In light of this, we will now prove Theorem~\ref{thm:sensing} {\bf with} the new QR-based iterates \eqref{eq:update_k1}.

\begin{lemma}\label{lemma:sensing-updateeqn1}
  Let $\Uht$ be the $t$-th step iterate of \as and let $\Ut, \widehat{V}^{t+1}$ and $\Vt$ be obtained by Update \eqref{eq:update_k1}. Then, 
\begin{align}
  \widehat{V}^{t+1} = \underbrace{V^*\Sigma^*{U^*}^\dag  U^{t}}_{\substack{\text{Power-method}\\ \text{Update}}} - \underbrace{F}_{\substack{\text{Error}\\ \text{Term}}},\ \ 
  V^{t+1}=\widehat{V}^{t+1}(R^{(t+1)})^{-1},\label{eq:update_k}
\end{align}
where $F$ is an error matrix defined in \eqref{eq:F_k} and $R^{(t+1)}$ is a triangular matrix obtained using QR-decomposition of $\widehat{V}^{t+1}$. 
\end{lemma}
See Appendix~\ref{app:sensing} for a detailed proof of the above lemma. 

Before we give an expression for the error matrix $F$, we define the following notation.
Let $v^* \in \R^{nk}$ be given by: $v^*=vec(V^*)$, i.e., $v^*=\left[v_1^{*\dag} v_2^{*\dag}\dots v_k^{*\dag}\right]^\dag$. 
Define $B$, $C$, $D$, $S$ as follows: {\small 
\begin{align}\label{eq:BCD_k}
& B \eqdef \left[\begin{array}{ccc}
                 B_{11} & \cdots & B_{1k} \\
		 \vdots  	& \ddots &\vdots \\
                 B_{k1} & \cdots & B_{kk} \\
                \end{array} \right]
\mbox{ , }
 C \eqdef \left[\begin{array}{ccc}
                 C_{11} & \cdots & C_{1k} \\
		 \vdots  	& \ddots &\vdots \\
                 C_{k1} & \cdots & C_{kk} \\
                \end{array} \right]\mbox{ , }\nonumber\\
&D\eqdef \left[\begin{array}{ccc}
                 D_{11}  & \cdots & D_{1k} \\
		 \vdots 	& \ddots &\vdots \\
                 D_{k1} & \cdots & D_{kk} \\
                \end{array} \right]\mbox{ , }
S\eqdef \left[\begin{matrix}\sigma_1^*I_{n}&\dots&0_{n}\\\vdots&\ddots&\vdots\\0_n&\dots&\sigma_k^*I_n\end{matrix}\right]. 
\end{align}}
where , for $1\leq p,q\leq k$: $B_{pq} \eqdef \sum_{i=1}^d A_i u_p^{t}{u_q^{t}}^\dag  A_i^\dag $,\\ $C_{pq} \eqdef \sum_{i=1}^d A_i u_p^{t}{u_q^*}^\dag  A_i^\dag, $ and, $D_{pq} \eqdef \iprod{u_p^{t}}{u_q^*}\mathbb{I}_{n\times n}.$ Recall that, $ u_p^{t}$ is the $p$-th column of $U^{t}$ and $u_q^*$ is the $q$-th left singular vector of the underlying matrix $M=U^*\Sigma^* (V^*)^\dag$. 
Finally $F$ is obtained by ``de-stacking'' the vector\\ $B^{-1}\left(BD-C\right)Sv^*$ i.e., the $i^{\textrm{th}}$ column of $F$ is given by:{\small
\begin{align}
  F_i&\eqdef\left[\begin{matrix} \left(B^{-1}\left(BD-C\right)Sv^*\right)_{ni+1} \\ \left(B^{-1}\left(BD-C\right)Sv^*\right)_{ni+2} \\ \vdots \\ \left(B^{-1}\left(BD-C\right)Sv^*\right)_{ni+n}
            \end{matrix}\right], F \eqdef \left[F_1 \; F_2 \;\cdots \;F_k \right]. \label{eq:F_k}
\end{align}}
Note that the notation above should have been $B^t,C^t$ and so on. We suppress the dependence on $t$ for notational simplicity. 
Now, from Update \eqref{eq:update_k}, we have
\begin{align}
  &V^{t+1} = \widehat{V}^{t+1}{R^{(t+1)}}^{-1} = \left(V^*\Sigma^*{U^*}^\dag  U^{t} - F\right){R^{(t+1)}}^{-1} \nonumber\\
  \Rightarrow & {V_{\perp}^*}^\dag V^{t+1} = -{V_{\perp}^*}^\dag  F {R^{(t+1)}}^{-1}.\label{eqn:sensing-orthogonaliter-error}
\end{align}
where $V_{\perp}^*$ is an orthonormal basis of  $\textrm{Span}\left(v_1^*,v_2^*,\cdots,v_k^*\right)^{\perp}$. Therefore, 
\begin{multline*}
dist(V^*, V^{t+1})=\|{V_{\perp}^*}^\dag V^{t+1}\|_2=\|{V_{\perp}^*}^\dag  F {R^{(t+1)}}^{-1}\|_2\leq \|F(\Sigma^*)^{-1}\|_2\|\Sigma^*{R^{(t+1)}}^{-1}\|_2.\end{multline*}
Now, we break down the proof of Theorem \ref{thm:sensing} into the following two steps: 
\begin{itemize}\setlength{\topsep}{0pt}\setlength{\itemsep}{0pt}
  \item show that $\twonorm{F(\Sigma^*)^{-1}}$ is  small (Lemma \ref{lemma:sensing-twonorm-F}) and
  \item show that $\|\Sigma^*{R^{(t+1)}}^{-1}\|_2$ is small(Lemma \ref{lemma:sensing-minsingval-Rt}).
\end{itemize}

We will now state the two corresponding lemmas.
Complete proofs can be found in Appendix \ref{app:sensingk}
The first lemma bounds the spectral norm of $F(\Sigma^*)^{-1}$.
\begin{lemma}\label{lemma:sensing-twonorm-F}
Let linear measurement $\aff$ satisfy RIP for all $2k$-rank matrices and let $b=\aff(M)$ with $M\in \rmn$ being a rank-$k$ matrix. Then, spectral norm of error matrix $F(\Sigma^*)^{-1}$ (see Equation~\ref{eq:update_k}) after $t$-th iteration update satisfy:
\begin{align}\label{eqn:sensing-twonorm-F}
  \twonorm{F(\Sigma^*)^{-1}} \leq \frac{\delta_{2k} k }{1-\delta_{2k}}dist(U^t, U^*).
\end{align}
\end{lemma}
The following lemma bounds the spectral norm of $\Sigma^*R^{(t+1)^{-1}}$.
\begin{lemma}\label{lemma:sensing-minsingval-Rt}
Let linear measurement $\aff$ satisfy RIP for all $2k$-rank matrices and let $b=\aff(M)$ with $M\in \rmn$ being a rank-$k$ matrix. Then, 
\begin{equation}\label{eqn:sensing-minsingval-Rt}
  \|\Sigma^*(R^{(t+1)})^{-1}\|_2\leq \frac{\sigma^*_1/\sigma^*_k}{\sqrt{1-\dist^2\left(U^{t},U^*\right)}-\frac{(\sigma^*_1/\sigma^*_k)\delta_{2k}k\dist(U^{t},U^*)}{1-\delta_{2k}}}.
\end{equation}
\end{lemma}
With the above two lemmas, we now prove Theorem \ref{thm:sensing}.
\begin{proof}[Proof Of Theorem \ref{thm:sensing}]
  Using \eqref{eqn:sensing-orthogonaliter-error}, \eqref{eqn:sensing-twonorm-F} and \eqref{eqn:sensing-minsingval-Rt}, we
obtain the following:
\begin{align}
  \dist\left(V^{t+1},V^*\right) &= \twonorm{{V_{\perp}^*}^\dag V^{t+1}},\nonumber\\
	&\leq \twonorm{{V_{\perp}^*}^\dag F(\Sigma^*)^{-1}\Sigma^*{R^{(t+1)}}^{-1}}, \nonumber\\
	&\leq \twonorm{V_{\perp}^*}\twonorm{F(\Sigma^*)^{-1}}\twonorm{\Sigma^*{R^{(t+1)}}^{-1}} \nonumber\\
	&\leq \frac{(\sigma^*_1/\sigma^*_k)\delta_{2k} k \cdot \dist\left(U^{t},U^*\right)}{(1-\delta_{2k})L},\label{eq:distb_sensing}		
\end{align}
where $L= \sqrt{1-\dist\left(U^{t},U^*\right)^2} - 
			\frac{(\sigma^*_1/\sigma^*_k)\delta_{2k} k \dist\left(U^{t},U^*\right)}{1-\delta_{2k}}$. 
Also, note that $U^{0}$ is obtained using SVD of $\sum_i A_i b_i$. Hence, using Lemma~\ref{lem:jmd}, we have:
\begin{align}
& \|\aff(U^{0}\Sigma^{0}V^{0}-U^*\Sigma^*(V^*)^\dag\|_2^2\leq 4\delta_{2k}\|\aff(U^*\Sigma^*(V^*)^\dag)\|_2^2,\nonumber\\
\Rightarrow& \|U^{0}\Sigma^{0}V^{0}-U^*\Sigma^*(V^*)^\dag\|_F^2\leq 4\delta_{2k}(1+3\delta_{2k})\|\Sigma^*\|_F^2,\nonumber\\
\Rightarrow&\|U^{0}(U^{0})^\dag U^*\Sigma^*(V^*)^\dag-U^*\Sigma^*(V^*)^\dag\|_F^2\leq 6\delta_{2k}\|\Sigma^*\|_F^2,\nonumber\\
\Rightarrow&(\sigma_k^*)^2\|(U^{0}(U^{0})^\dag-I)U^*\|_F^2\leq 6\delta_{2k}k(\sigma_1^*)^2,\nonumber\\
\Rightarrow&\dist(U^{0}, U^*)\leq \sqrt{6\delta_{2k}k}\left(\frac{\sigma_1^*}{\sigma_k^*}\right)<\frac{1}{2}.
\end{align}
Using \eqref{eq:distb_sensing} with  $\dist\left(U^{0},U^*\right) < \frac{1}{2}$ and $\delta_{2k} < \frac{1}
{24(\sigma^*_1/\sigma^*_k)^2k}$, we obtain: $\dist\left(V^{t},V^*\right) < \frac{1}{4} \dist\left(U^{t},U^*\right)$.  
Similarly we can show that $\dist\left(U^{t+1},U^*\right) < \frac{1}{4} \dist\left(V^{t},V^*\right)$. 
\end{proof}

\newcommand{\uht}{\widehat{u}^t}
\section{Matrix Completion}
In this section, we study the Matrix Completion problem \eqref{eq:matcomp} and show that, assuming $k$ and $\frac{\so_1}{\so_k}$ are constant, \amc (Algorithm~\ref{algo:completion-altmin}) recovers the underlying matrix $M$ using only $O(n \log n)$ measurements (i.e., we prove Theorem~\ref{thm:sampling-altmin-main}).

As mentioned, while observing elements in $\Omega$ constitutes a linear map, matrix completion is different from matrix sensing because the map does not satisfy RIP. The (now standard) approach is to assume incoherence of the true matrix $M$, as done in Definition \ref{defn:sampling-incoherentmatrices}. With this, and the random sampling of $\Omega$, matrix completion exhibits similarities to matrix sensing. For our analysis, we can again use the fact that incoherence allows us to view alternating minimization as a perturbed power method, whose error we can control.

However, there are important differences between the two problems, which make the analysis of completion more complicated. Chief among them is the fact that we need to establish {\em the incoherence of each iterate}. For the first initialization $\widehat{U}^0$, this necessitates the ``clipping" procedure (described in step 4 of the algorithm). For the subsequent steps, this requires the partitioning of the observed $\Omega$ into $2T+1$ sets (as described in step 2 of the algorithm).

As in the case of matrix sensing, we prove our main result for matrix completion (Theorem~\ref{thm:sampling-altmin-main}) by first establishing a geometric decay of the distance between the subspaces spanned by $\widehat{U}^{t}, \widehat{V}^{t}$ and $U^*,V^*$ respectively.
\begin{theorem}\label{thm:sampling-altmin-geomconv}
 Under the assumptions of Theorem \ref{thm:sampling-altmin-main}, the $(t+1)^{\textrm{th}}$ iterates $\Uw^{t+1}$ and $\Vw^{t+1}$
satisfy the following property w.h.p.:
\begin{align*}
 \dist\left(\Vw^{t+1}, V^*\right) &\leq \frac{1}{4} \dist\left(\Uw^{t}, U^*\right)  \mbox{ and } \\
 \dist\left(\Uw^{t+1}, U^*\right) &\leq \frac{1}{4} \dist\left(\Vw^{t+1}, V^*\right),\;\;\;\; \forall \;\;1 \leq t \leq T.
\end{align*}
\end{theorem}
We use the above result along with incoherence of $M$ to prove Theorem~\ref{thm:sampling-altmin-main}. See Appendix~\ref{app:mc} for a detailed proof. 

Now, similar to the matrix sensing case, alternating minimization needs an initial iterate that is close enough to $U^*$ and $V^*$, from
where it will then converge. To this end, Steps $3-4$ of Algorithm \ref{algo:completion-altmin}  use SVD of $P_\Omega(M)$ followed by clipping to initialize $\widehat{U}^0$. While the SVD step guarantees that  $\widehat{U}^0$ is close enough to $U^*$, it might not remain incoherent. To maintain incoherence, we introduce an extra clipping step which guarantees incoherence of $\widehat{U}^0$ while also ensuring that $\widehat{U}^0$ is close enough to $U^*$ (see Lemma~\ref{lemma:svd-clipping-jointguarantee})
\begin{lemma}\label{lemma:svd-clipping-jointguarantee}
Let $M,\Omega,p$ be as defined in Theorem~\ref{thm:sampling-altmin-main}. Also, let $U^0$ be the initial iterate obtained by  step $4$ of Algorithm \ref{algo:completion-altmin}. Then, w.h.p. we have
\begin{itemize}\setlength{\topsep}{-5pt}\setlength{\itemsep}{-5pt}
  \item	$\dist\left(U^{0},U^*\right) \leq \frac{1}{2}$ and
  \item	$U^{0}$ is incoherent with parameter $4\mu \sqrt{k}$.
\end{itemize}
\end{lemma}
The above lemma guarantees a ``good'' starting point for alternating minimization. Using this, we now present a proof of Theorem~\ref{thm:sampling-altmin-geomconv}.  
Similar to the sensing section, we first explain key ideas of our proof using rank-$1$ example. Then in Section~\ref{sec:comp_general} we extend our proof to  general rank-$k$ matrices. 
\subsection{Rank-$1$ Case}
Consider the rank-$1$ matrix completion problem where $M=\so\uo(\vo)^\dag$.
Now, the $t$-th step iterates $\vht$ of Algorithm \ref{algo:completion-altmin} are given by: 
\begin{align*}
  \vht=\argmin_{\widehat{v}} \sum_{(i,j)\in \Omega}(M_{ij}-\widehat{u}^t_i\widehat{v}_j)^2.
\end{align*}
Let $\ut=\uht/\|\uht\|_2$. Then, $\forall j$:{\small
\begin{align}
\hspace*{-10pt}&\|\uht\|_2\sum_{i:(i,j)\in \Omega}(\ut_i)^2\vht_j =\so\sum_{i:(i,j)\in\Omega}\ut_i\uo_i\vo_j \nonumber\\
&\Rightarrow \|\uht\|_2\vht_j = \frac{\so}{\sum_{i:(i,j)\in \Omega}(\ut_i)^2}\sum_{i:(i,j)\in\Omega}\ut_i\uo_i\vo_j\nonumber\\
&=\sigma^*\ip{\ut}{\uo}\vo_j-\frac{\sigma^*(\ip{\ut}{\uo}\sum_{i:(i,j)\in\Omega}(\ut_i)^2\vo_j-\sum_{i:(i,j)\in\Omega}\ut_i\uo_i\vo_j)}{\sum_{i:(i,j)\in\Omega}(\ut_i)^2}.\label{eq:r1compupdate}
\end{align}}
Hence,
\begin{align}
\hspace*{-5pt}  \|\uht\|_2\vht=\underbrace{\ip{\uo}{\ut}\sigma^*\vo}_{\text{Power Method}}-\underbrace{\sigma^*B^{-1}\left(\ip{\ut}{\uo}B-C\right)\vo}_{\text{Error Term}},
\label{eq:compvtupdate}
\end{align}
where $B, C\in \mathbb{R}^{n\times n}$ are diagonal matrices, such that, 
\begin{equation}
  \label{eq:bc_comp}
  B_{jj}=\frac{\sum_{i:(i,j)\in \Omega}(\ut_i)^2}{p}, \ \ \ C_{jj}=\frac{\sum_{i:(i,j)\in\Omega}\ut_i\uo_i}{p}. 
\end{equation}
Note the similarities between the update \eqref{eq:compvtupdate} and  the rank-$1$ update \eqref{eq:error_1} for the sensing case. Here again, it is essentially a power-method update (first term) along with a bounded error term (see Lemma~\ref{lem:errb_comp1}). Using this insight, we now prove Theorem~\ref{thm:sampling-altmin-geomconv} for the special case of rank-$1$ matrices. 
Our proof can be divided in three major steps:
\begin{itemize}\setlength{\topsep}{-5pt}\setlength{\itemsep}{-5pt}
\item {\em Base Case}: Show that $u^0=\widehat{u}^0/\|\widehat{u}^0\|_2$ is incoherent and have small distance to $\uo$ (see Lemma~\ref{lemma:svd-clipping-jointguarantee}).  
\item {\em Induction Step (distance)}: Assuming $\ut=\widehat{u}^t/\|\uht\|_2$ to be incoherent and that $\ut$ has a small distance to $\uo$, $\vt$ decreases distances to $\vo$ by at least a constant factor. 
\item {\em Induction Step (incoherence)}: Show incoherence of $\vt$, while assuming incoherence of $\ut$ (see Lemma~\ref{lem:incoherence_vt})
\end{itemize}
We first prove the second step of our proof. To this end, we provide the following lemma that bounds the error term. See Appendix~\ref{app:rank1_comp} for a proof of the below given lemma. 
\begin{lemma}\label{lem:errb_comp1}
Let $M$, $p$, $\Omega$, $\ut$ be as defined in Theorem~\ref{thm:sampling-altmin-main}. Also, let $\ut$ be a unit vector with incoherence parameter $\mu_1=\frac{6(1+\delta_2)\mu}{1-\delta_2}$.
Then, w.p. at least $1-\frac{1}{n^3}$: 
$$\|B^{-1}\left(\ip{\uo}{\ut}B-C\right)\vo\|_2\leq \frac{\delta_{2}}{1-\delta_{2}}\sqrt{1-\ip{\ut}{\uo}^2}.$$
\end{lemma}
Multiplying \eqref{eq:compvtupdate} with $\vo$ and using Lemma~\ref{lem:errb_comp1}, we get:
\begin{equation}
  \label{eq:ipvtvo_comp}
  \|\uht\|_2\ip{\widehat{v}^{t+1}}{\vo}\geq \so\ip{\ut}{\uo}-2\so\delta_2\sqrt{1-\ip{\ut}{\uo}^2},
\end{equation}
where $\delta_2<\frac{1}{12}$ is a constant defined in the Theorem statement and is similar to the RIP constant in Section \ref{sec:sensing}.

Similarly, by multiplying \eqref{eq:compvtupdate} with $v_{\perp}$ (where $\ip{v_{\perp}^*}{\vo}=0$ and $\|v_{\perp}^*\|_2=1$) and using Lemma~\ref{lem:errb_comp1}:
$$\|\uht\|_2\ip{\widehat{v}^{t+1}}{v_{\perp}^*}\leq 2\so\delta_2\sqrt{1-\ip{\ut}{\uo}^2}.$$
Using the above two equations: 
\begin{equation*}1-\ip{\vt}{\vo}^2\leq \frac{4\delta_2^2( 1-\ip{\ut}{\uo}^2)}{(\ip{\ut}{\uo}-2\delta_2\sqrt{1-\ip{\ut}{\uo}^2})^2+(2\delta_2\sqrt{1-\ip{\ut}{\uo}^2})^2}.\end{equation*}
Assuming, $\ip{\vt}{\vo}\geq 6\delta_2$, $$\dist(\vt, \vo)=\sqrt{1-\ip{\vt}{\vo}^2}\leq \frac{1}{4}\sqrt{1-\ip{\ut}{\uo}^2}.$$
Using same arguments, we can show that, $\dist(u^{t+1}, \uo)\leq \dist(\vt, \vo)/4$. Hence, after $O(\log(1/\epsilon))$ iterations, $\dist(\ut, \uo)\leq \epsilon $ and $\dist(\vt, \vo)\leq \epsilon$. This proves our second step. 

We now provide the following lemma to prove the third step. We stress that $\vt$ does not increase the incoherence parameter ($\mu_1$) when compared to that of $\ut$. 
\begin{lemma}  \label{lem:incoherence_vt}
Let $M$, $p$, $\Omega$ be as defined in Theorem~\ref{thm:sampling-altmin-main}. Also, let $\ut$ be a unit vector with incoherence parameter $\mu_1=\frac{6(1+\delta_2)\mu}{1-\delta_2}$. 
 Then, w.p. at least $1-\frac{1}{n^3}$, $\vt$ is also $\mu_1$ incoherent. 
\end{lemma}
See Appendix~\ref{app:rank1_comp} for a detailed proof of the lemma. 

Finally, for the base case we need that $u^0$ is $\mu_1$ incoherent and also $\ip{u^0}{\uo}\geq 6\delta_2$. This follows directly by using Lemma~\ref{lemma:svd-clipping-jointguarantee} and the fact that $\delta_2\leq 1/12$. 

Note that, to obtain an error of $\epsilon$, \amc needs to run for $O\left(\log \frac{\|M\|_F}{\epsilon}\right)$ iterations. Also, we need to sample a fresh $\Omega$ at each iteration of \amce. Hence, the total number of samples needed by \amc is $O\left(\log \frac{\|M\|_F}{\epsilon}\right)$ larger than the  number of samples required per step.
\subsection{Rank-$k$ case}\label{sec:comp_general}
We now extend our proof of Theorem~\ref{thm:sampling-altmin-geomconv}  to matrices with arbitrary rank. Here again, we show that the \amc algorithm reduces to power method with bounded perturbation at each step. 

Similar to the matrix sensing case, we analyze the following QR decomposition based update instead of directly analyzing the updates of Algorithm~\ref{algo:completion-altmin}:
\begin{align}
\widehat{U}^{t}&= \, U^{t}R_U^{t}\ \ \text{ (QR decomposition), }\nonumber\\
\widehat{V}^{t+1}&=\argmin_{\widehat{V}} \  \|P_{\Omega}(U^t\widehat{V}^\dag)-P_{\Omega}(M)\|_F^2,\nonumber\\ 
\widehat{V}^{t+1}&=V^{t+1}R_V^{t+1}.\ \ \text{ (QR decomposition) },\nonumber\\ 
\widehat{U}^{t+1}&=\argmin_{\widehat{U}} \  \|P_{\Omega}(\widehat{U}(V^{t+1})^\dag)-P_{\Omega}(M)\|_F^2.\label{eq:updatecomp_k1}
\end{align}
Here again, we would stress that the updates output exactly the same matrices at the end of each iteration and we prefer QR-based updates due to notational ease. 

Now, as matrix completion is a special case of matrix sensing, Lemma~\ref{lemma:sensing-updateeqn1} characterizes the updates of the \amc algorithm (see Algorithm~\ref{algo:completion-altmin}). That is, 
\begin{align}
  \widehat{V}^{t+1} &= \underbrace{V^*\Sigma^*{U^*}^\dag  U^{t}}_{\text{Power-method Update}} - \underbrace{F}_{\text{Error Term}}, \nonumber\\
  V^{t+1}&=\widehat{V}^{t+1}(R^{(t+1)})^{-1},\label{eq:updatecomp_k}
\end{align}
where $F$ is the error matrix defined in \eqref{eq:F_k} and $R^{(t+1)}$ is a upper-triangular matrix obtained using QR-decomposition of $\widehat{V}^{t+1}$. See \eqref{eq:BCD_k} for the definition of  $B, C$, $D$, and $S$. 

Also, note that for the special case of matrix completion, $B_{pq}, C_{pq}, 1\leq p,q\leq k$ are {\em diagonal matrices} with $$(B_{pq})_{jj}=\frac{1}{p}\sum_{i:(i,j)\in \Omega}\Ut_{ip}\Ut_{iq},\ \ \ (C_{pq})_{jj}=\frac{1}{p}\sum_{i:(i,j)\in \Omega}\Ut_{ip}\Uo_{iq}.$$ We use this structure to further simplify the update equation. We first define matrices $B^j, C^j, D^j\in \mathbb{R}^{k\times k}, 1\leq i\leq n$: {\small $$B^j=\frac{1}{p}\sum_{i:(i,j)\in \Omega} (\Ut)^{(i)}{(\Ut)^{(i)}}^\dag , \ \ C^j=\frac{1}{p}\sum_{i:(i,j)\in\Omega} (\Ut)^{(i)}{(\Uo)^{(i)}}^\dag,$$} and $D^j=(\Ut)^\dag \Uo$. 
 Using the above notation, \eqref{eq:updatecomp_k} decouples into $n$ equations of the form ($1\leq j\leq n$):
 \begin{align}\label{eqn:sampling-updateeqn-decoupled}
    (\Vt)^{(j)} = (\Vo)^{(j)} (D^j-(B^j)^{-1}(B^jD^j-C^j)) (R^{(t+1)})^{-1},
 \end{align}
 where $(\Vt)^{(j)}$ and $(\Vo)^{(j)}$ denote the $j^{\textrm{th}}$ rows of $\Vt$ and $\Vo$ respectively.

Using the above notation, we now provide a proof of Theorem~\ref{thm:sampling-altmin-geomconv} for the general rank-$k$ case. 
\begin{proof}[Proof of Theorem~\ref{thm:sampling-altmin-geomconv}]
Multiplying the update equation \eqref{eq:updatecomp_k} on the left by $\left(\Vo_{\perp}\right)^{\dag}$, we get:\\
$(\Vo_{\perp})^{\dag}\Vht=-(\Vo_{\perp})^{\dag}F (\Rt)^{-1}.$ That is, 
\begin{align*}\dist(V^*, V^{t+1})&=\|{V_{\perp}^*}^\dag V^{(t+1)}\|_2=\|{V_{\perp}^*}^\dag  F {R^{(t+1)}}^{-1}\|_2\\&\leq \|F(\Sigma^*)^{-1}\|_2\|\Sigma^*{R^{(t+1)}}^{-1}\|_2.\end{align*}
Now, similar to the sensing case (see Section~\ref{sec:sensing_k}) we break down our proof into the following two steps:
\begin{itemize}
  \item Bound $\twonorm{F(\Sigma^*)^{-1}}$  (Lemma \ref{lemma:Fbound_comp}) and
  \item Bound $\|\Sigma^*{R^{(t+1)}}^{-1}\|_2$, i.e., the minimum singular value of $\left(\Sigma^*\right)^{-1}R^{(t+1)}$ (Lemma \ref{lemma:Rbound_comp}).
\end{itemize}
Using Lemma~\ref{lemma:Fbound_comp} and Lemma~\ref{lemma:Rbound_comp}, w.p. at least $1-1/n^3$, 
\begin{equation*}\dist(V^*, V^{t+1})\leq \|F(\Sigma^*)^{-1}\|_2\|\Sigma^*{R^{(t+1)}}^{-1}\|_2 \leq  \frac{(\sigma^*_1/\sigma^*_k)k(\delta_{2k}/(1-\delta_{2k}))  \cdot \dist\left(U^{(t)},U^*\right)}{\sqrt{1-\dist\left(U^{(t)},U^*\right)^2} - 
			\frac{(\sigma^*_1/\sigma^*_k)k\delta_{2k}  \dist\left(U^{(t)},U^*\right)}{1-\delta_{2k}}}.\end{equation*}
Now, using Lemma~\ref{lemma:svd-clipping-jointguarantee} we get: $\dist(U^t, U^*)\leq \dist(U^0, U^*)\leq \frac{1}{2}$. By selecting $\delta_{2k}<\frac{\so_k}{12k\so_1}$, i.e., $p\geq \frac{C(\so_1)^2k^4\log n}{m(\so_k)^2}$ and using above two inequalities:
$$\dist(\Vt, \Vo)\leq \frac{1}{4}\dist(\Ut, \Uo).$$
Furthermore, using Lemma~\ref{lemma:completion-incoherence-Vt} we get that $\Vt$ is $\mu_1$ incoherent. 
Hence, using similar arguments as above, we also get:
$\dist(U^{t+1}, \Uo)\leq \left(\frac{1}{4}\right)\dist(\Vt, \Vo).$ 
\end{proof}
We now provide lemmas required by our above given proof. See Appendix~\ref{app:compk} for detailed proof of each of the lemmas.

We first provide a lemma to bound incoherence of $\Vt$, assuming incoherence of $\Ut$. 
\begin{lemma}\label{lemma:completion-incoherence-Vt}
Let $M, \Omega, p$ be as defined in Theorem~\ref{thm:sampling-altmin-main}. Also, let $U^t$ be the $t$-th step iterate obtained by \eqref{eq:updatecomp_k1}. Let $U^t$  be $\mu_1=\frac{16\sigma_1^*\mu\sqrt{k}}{\sigma_k^*}$ incoherent. Then, w.p. at least $1-1/n^3$, iterate $V^{(t+1)}$ is also $\mu_1$ incoherent.
\end{lemma}
We now bound the error term ($F$) in  AltMin update \eqref{eq:updatecomp_k}. 
\begin{lemma}\label{lemma:Fbound_comp}
  Let $F$ be the error matrix defined by \eqref{eq:F_k} (also see \eqref{eq:updatecomp_k}) and let $\Ut$ be a $\mu_1$-incoherent orthonormal matrix obtained after $(t-1)^{\textrm{th}}$ update. Also, let $M$, $\Omega$, and $p$ satisfy assumptions of Theorem~\ref{thm:sampling-altmin-main}. Then, w.p. at least $1-1/n^3$: 
$$\twonorm{F(\Sigma^*)^{-1}} \leq \frac{\delta_{2k} k }{1-\delta_{2k}}dist(U^t, U^*).$$
\end{lemma}
Next, we present a lemma to bound $\|(\Rt)^{-1}\|_2$. 
\begin{lemma}\label{lemma:Rbound_comp}
  Let $\Rt$ be the lower-triangular matrix obtained by QR decomposition of $\Vht$ ( see \eqref{eq:updatecomp_k}) and let $\Ut$ be a $\mu_1$-incoherent orthonormal matrix obtained after $(t-1)^\textrm{th}$ update. Also, let $M$ and $\Omega$ satisfy assumptions of Theorem~\ref{thm:sampling-altmin-main}. Then,
\begin{align}\label{eqn:completion-minsingval-Rt}
  &\|\Sigma^*(R^{(t+1)})^{-1}\|_2\leq \frac{\sigma^*_1/\sigma^*_k}{\sqrt{1-\dist^2\left(U^{(t)},U^*\right)}-\frac{(\sigma^*_1/\sigma^*_k)\delta_{2k}k\dist(U^{(t)},U^*)}{1-\delta_{2k}}}
\end{align}
\end{lemma}
\begin{proof}
Lemma follows by exactly the same proof as that of Lemma~\ref{lemma:sensing-minsingval-Rt} for the matrix sensing case. 
\end{proof}
\newcommand{\soj}{(\so_j)^2}
\newcommand{\soi}{(\so_i)^2}
\newcommand{\soii}{(\so_{i+1})^2}
\newcommand{\Up}{U_{\perp}}
\begin{algorithm}[t]
\caption{{\bf Stage-AltMin}: Stagewise Alternating Minimization for Matrix Sensing }
\begin{algorithmic}[1]
\STATE  Input: $b,\aff$
\STATE $\widehat{U}^T \leftarrow []$, $\widehat{V}^T \leftarrow []$
\FOR{$i=1,\cdots,k$}
{\STATE {$[\widehat{U}^0_{1:i}\ \widehat{V}^0_{1:i}]=\mbox{top i-singular vectors of }$ $\left(\widehat{U}^T_{1:i-1} (\widehat{V}_{1:i-1}^T)^\dag-\frac{3}{4}\aff^T(\aff(\widehat{U}_{1:i-1}^T(\widehat{V}^T_{1:i-1})^\dag)-b)\right)$ i.e., one step of SVP \cite{JainMD10}\label{stp:sam_svp}}}
\FOR{$t=0,\cdots,T-1$}
	\STATE {$\widehat{V}_{1:i}^{t+1}\leftarrow \argmin_{V\in \R^{n\times i}} \  \|\aff(\widehat{U}^t_{1:i}V^\dag )-b\|_2^2$\label{stp:vt}}
	\STATE {$\widehat{U}_{1:i}^{t+1}\leftarrow \argmin_{U\in \mathbb{R}^{m\times i}} \  \|\aff(U_{1:i}(\widehat{V}_{1:i}^{t+1})^\dag )-b\|_2^2$\label{stp:ut}}
\ENDFOR
\ENDFOR
\STATE Output: $X=\widehat{U}_{1:i}^T(\widehat{V}_{1:i}^T)^\dag $
\end{algorithmic}\label{algo:sensing-altmin-improved}
\end{algorithm}
\section{Stagewise AltMin Algorithm}\label{sec:sam}
In Section~\ref{sec:sensing}, we showed that if $\delta_{2k}\leq \frac{(\so_k)^2}{(\so_1)^2k}$ then \as (Algorithm~\ref{algo:sensing-altmin}) recovers the underlying matrix. This means that, $d=\frac{(\so_1)^4}{(\so_k)^4}k^2n\log  n$ random Gaussian measurements (assume $m\leq n$) are required to recover $M$. For matrices with large condition number ($\so_1/ \so_k$), this would be significantly larger than the information theoretic bound of $O(k n \log n/k)$ measurements. 

To alleviate this problem, we present a modified version of \as called Stage-AltMin. Stage-AltMin proceeds in $k$ stages where in the $i$-th stage, a rank-$i$ problem is solved. The goal of the $i$-th stage is to recover top $i$-singular vectors of $M$, up to $O(\so_{i+1})$  error. 

Specifically, we initialize the $i$-th stage of our algorithm using one step of the SVP algorithm \cite{JainMD10} (see Step~\ref{stp:sam_svp} of Algorithm~\ref{algo:sensing-altmin-improved}).  We then show that, if $\delta_{2k}\leq \frac{1}{10k}$, then Stage-AltMin (Steps~\ref{stp:vt}, \ref{stp:ut} of Algorithm~\ref{algo:sensing-altmin-improved}) decreases the error $\|M-\widehat{U}^T_{1:i}(\widehat{V}^T_{1:i})^\dag\|_F$ to $O(\so_{i+1})$. Hence, after $k$ steps, the error decreases to $O(\so_{k+1})=0$. Note that, $\widehat{U}^t_{1:i}\in \mathbb{R}^{m\times i}$ represents the $t$-th step iterate ($U$) in the $i$-th stage; $\widehat{V}^t_{1:i}\in \mathbb{R}^{n\times i}$ is also defined similarly.

Recall that, the main problem with our analysis of \as is that if $\sigma_i\gg \sigma_{i+1}$ (for some $i$) then $\delta_{2k}\leq \frac{(\so_{i+1})^2}{(\so_i)^2k}$ would need to be  small. However, in such a scenario, 
 the $i$-th stage of Algorithm~\ref{algo:sensing-altmin-improved} can be thought of as solving a noisy sensing problem where the goal is to recover $M_i\eqdef\Uo_{1:i}\So_{1:i}(\Vo_{1:i})^\dag$ using noisy measurements $b=\aff(\Uo_{1:i}\So_{1:i}(\Vo_{1:i})^\dag+N)$ where noise matrix $N\eqdef\Uo_{i+1:k}\So_{i+1:k}(\Vo_{i+1:k})^\dag$.
Here $M_i$ and $N$ represent the top $i$ singular components and last $k-i$ singular components of $M$ respectively.
Hence, using noisy-case type analysis (see Section~\ref{app:noisy_sensing}) we show that the error $\|M-\widehat{U}^t(\widehat{V}^t)^\dag\|_F$ decreases to $O(\so_{i+1})$. 

We now formally present the proof of our main result (see Theorem~\ref{thm:sam_err}). 
\begin{proof}[Proof Of Theorem~\ref{thm:sam_err}]
We prove the theorem using mathematical induction. 

\noindent {\bf Base Case}: After the $0$-th step, error is: $\|M\|_F^2\leq \sum_{j=1}^k \sigma_j^2\leq k\sigma_1^2$. Hence, base case holds. 

\noindent {\bf Induction Step}: Here, assuming that the error bound holds for $(i-1)$-th stage, we prove the error bound for the $i$-th stage. 

Our proof proceeds in two steps. First, we show that the initial point $\Uw^0_{1:i}$, $\Vw^0_{1:i}$ of the $i$-th stage, obtained using Step~\ref{stp:sam_svp}, has $c(\so_{i})^2+O\left(k(\so_{i+1})^2\right)$ error, with $c<1$. In the second step,  we show that using the initial points  $\Uw^0_{1:i}$, $\Vw^0_{1:i}$, the AltMin algorithm iterations in the $i$-th stage (Steps~\ref{stp:vt}, \ref{stp:ut}) reduces the error to $\max(\epsilon,16k\sigma_{i+1}^2)$. 

We formalize the above mentioned first step in Lemma~\ref{lemma:sam_svp_err} and then prove the second step in Lemma~\ref{lemma:sam_altmin}. 
\end{proof}
We now present two lemmas used by the above given proof. See Appendix~\ref{app:stagewise_sensing} for a proof of each of the lemmas. 
\begin{lemma}\label{lemma:sam_svp_err}
Let assumptions of Theorem~\ref{thm:sam_err} be satisfied. Also, let $\Uw^0_{1:i}$, $\Vw^0_{1:i}$ be the output of Step~\ref{stp:sam_svp} of Algorithm~\ref{algo:sensing-altmin-improved}. Then, assuming that $\|M-\Uw^T_{1:i-1}\Vw^T_{1:i-1}\|_F^2\leq 16k(\so_{i})^2$, we obtain:
\begin{align*}
  \left\|M - \widehat{U}^0_{1:i}(\widehat{V}_{1:i}^0)^{\dag}\right\|_F^2 \leq \sum_{j=i+1}^k (\so_j)^2 + \frac{1}{100} (\so_{i})^2. 
\end{align*}
\end{lemma}
\begin{lemma}\label{lemma:sam_altmin}
Let assumptions of Theorem~\ref{thm:sam_err} be satisfied. Also, let $\Uw^T_{1:i}$, $\Vw^T_{1:i}$ be the $T$-th step iterates of the $i$-th stage of Algorithm~\ref{algo:sensing-altmin-improved}. Then, assuming that $\|M-\Uw^0_{1:i}V^0_{1:i}\|_F^2\leq \sum_{j=i+1}^k (\so_j)^2 + \frac{1}{100} (\so_{i})^2$, we obtain:
\begin{align*}
  \left\|M - \widehat{U}^T_{1:i}(\widehat{V}_{1:i}^T)^{\dag}\right\|_F^2 \leq \max(\epsilon, 16k(\so_{i+1})^2),
\end{align*}
where $T=\Omega(\log(\|M\|_F/\epsilon))$. 
\end{lemma}
\section{Summary and Discussion}








Alternating minimization provides an empirically appealing and popular approach to solving several different low-rank matrix recovery problems. The main motivation, and result, of this paper was to provide the {\bf first theoretical guarantees on the global optimality of alternating minimization}, for matrix completion and the related problem of matrix sensing. We would like to note the following aspects of our results and proofs:
\begin{itemize}\setlength{\itemindent}{0pt}\setlength{\topsep}{-5pt}\setlength{\itemsep}{-3pt}
\item For both the problems, we show that alternating minimization recovers the true matrix under {\em similar problem conditions} (RIP, incoherence) to those used by existing algorithms (based on convex optimization or iterated SVDs); computationally, our results show faster convergence to the global optima, but with possibly higher statistical (i.e. sample) complexity.
\item We develop a new framework for analyzing alternating minimization for low-rank problems. Key observation of our framework is that for some problems (under standard problem conditions) alternating minimization can be viewed as a perturbed version of the power method. In our case, we can control the perturbation error based on the extent of RIP / incoherence demonstrated by the problem. This idea is likely to have applications to other similar problems where  trace-norm based convex relaxation techniques have rigorous theoretical results but  alternating minimization has enjoyed more empirical success. For example, robust PCA \cite{ChandrasekaranSPW11,CandesLMW11}, spectral clustering \cite{JalaliCSX11} etc. 
\item Our analysis also sheds light on two key aspects of the  alternating minimization approach:\\
{\bf Initialization}: Due to its connection  to power method, it is now easy to see  that for alternating minimization to succeed, the initial iterate should not be orthogonal to the target vector. Our results indeed show that alternating minimization  succeeds if the
initial iterate is not ``almost orthogonal'' to the target subspace. This suggests that, selecting initial iterate smartly is preferable to random initialization.\\ 
{\bf Dependence on the condition number}: Our results for the alternating minimization algorithm depend on the condition
number. However, using a stagewise adaptation of alternating minimization, we can remove this dependence
for the matrix sensing problem. This suggests that (problem specific) modifications of the basic alternating minimization
algorithm may in fact perform better than the original one, while (mostly) retaining the computational / implementational
simplicity of the underlying method.
\end{itemize}

\onecolumn
\appendix
\section{Preliminaries}
\begin{lemma}[Lemma 2.1 of \cite{JainMD10}]\label{lem:jmd}
  Let $b=\aff(M)+e$, where $e$ is a bounded error vector, $M$ is a rank-$k$ matrix and $\aff$ is a linear measurement operator that satisfies $2k$-RIP with constant $\delta_{2k}$ (assume $\delta_{2k}<1/3$). Let $X^{t+1}$ be the $t+1$-th step iterate of SVP, then the following holds:
$$\|\aff(X^{t+1})-b\|_2^2\leq \|\aff(M)-b\|_2^2 + 2\delta_{2k}\|\aff(X^t)-b\|_2^2.$$
\end{lemma}
In our analysis, we heavily use the following two results. The first result is the well-known  Bernstein's inequality.
\begin{lemma}\label{lemma:bernstein}[Bernstein's inequality]
  Let $X_1,X_2,\cdots,X_n$ be independent random variables. Also, let $\left|X_i\right| \leq L \in \R\; \forall \; i$ w.p. $1$. Then, we have the following inequality:{\small
\begin{align}\label{eqn:bernstein}
  \prob{\left|\sum_{i=1}^n X_i - \sum_{i=1}^n \expec{X_i}\right| > t} \leq 2\exp\left(\frac{-t^2/2}{\sum_{i=1}^n \Var\left(X_i\right) + Lt / 3}\right).
\end{align}}
\end{lemma}
The second result is a restatement of Theorem 3.1 from \cite{KeshavanOM2009}.
\begin{theorem}\label{thm:KOM-firststepSVD}(Restatement of Theorem 3.1 from \cite{KeshavanOM2009})
  Suppose $M$ is an incoherent rank-$k$ matrix and let $p,\Omega$ be as in Theorem~\ref{thm:sampling-altmin-main}. Further, let $M_k$ be the best rank-$k$ approximation  of $\frac{1}{p}P_{\Omega}\left(M\right)$. Then, w.h.p. we have:
\begin{align}\label{eqn:KOM-firststepSVD}
  \twonorm{M - M_k} \leq C \sqrt{\frac{k}{p\sqrt{mn}}} \frob{M}.
\end{align}
\end{theorem}
\textbf{Remark}: Note that Theorem 3.1 from \cite{KeshavanOM2009} holds only for $T_r\left(P_{\Omega}(M)\right)$  where
$T_r\left(P_{\Omega}(M)\right)$ is a trimmed version of $P_{\Omega}(M)$ obtained by setting all rows and columns of $P_{\Omega}(M)$ with too many observed entries  to zero.
However, using standard Chernoff bound we can argue that for our choice of $p$, none of the rows and columns of $P_{\Omega}(M)$ have too many observed entries and hence $T_r\left(P_{\Omega}(M)\right)=P_{\Omega}(M)$, whp.
\section{Matrix Sensing: Proofs}\label{app:sensing}
The following is an alternate characterization of RIP that we use heavily in our proofs.
At a conceptual level, it says that if $\aff$ satisfies RIP, then it also preserves inner-product between any two rank-$k$ matrices (upto some additive error).

\begin{lemma}\label{lem:alt_rip}
Suppose $\calA(\cdot)$ satisfies $2k$-RIP with constant $\delta_{2k}$. Then, for any $U_1,U_2\in \R^{m\times k}$ and
$V_1,V_2\in \R^{n\times k}$, we have the following:
\begin{align}\label{eqn:sensing-2k-RIP2}
  \left| \left\langle \calA\left(U_1V_1^\dag \right), \calA\left(U_2V_2^\dag \right)\right\rangle  - \Tr \left(U_2^\dag U_1V_1^\dag V_2\right) \right|
    \leq 3\delta_{2k} \left\|U_1V_1^\dag \right\|_F \left\|U_2V_2^\dag \right\|_F
\end{align}
\end{lemma}
\begin{proof}
  Consider the matrices $X_1\eqdef U_1V_1^T$, $X_2\eqdef U_2V_2^T$ and $X=X_1+X_2$.
Since the rank of $X$ is at most $2k$, we obtain the following using the RIP of $\calA$:
\begin{align*}
  (1-\delta) \left\|U_1V_1^T + U_2V_2^T\right\|_F^2 \leq \left\|\calA(X)\right\|_2^2
	\leq (1+\delta) \left\|U_1V_1^T+U_2V_2^T\right\|_F^2.
\end{align*}
Concentrating on the second inequality, we obtain
\begin{align}
  &\sum_i \left(\Tr\left(A_i U_1V_1^T\right)+\Tr\left(A_i U_2V_2^T\right)\right)^2
	\leq (1+\delta) \left(\left\|U_1V_1^T\right\|_F^2 + \left\|U_2V_2^T\right\|_F^2 +
		\Tr\left(U_1V_1^TV_2U_2^T\right)\right) \nonumber \\
  \stackrel{(\zeta_1)}{\Rightarrow} &\sum_i \Tr\left(A_i U_1V_1^T\right)\Tr\left(A_i U_2V_2^T\right) - \Tr\left(U_1V_1^TV_2U_2^T\right)
	\leq \delta\left(\left\|U_1V_1^T\right\|_F^2 + \left\|U_2V_2^T\right\|_F^2 +
		\Tr\left(U_1V_1^TV_2U_2^T\right)\right) \nonumber \\
  \stackrel{(\zeta_2)}{\Rightarrow} &\sum_i \Tr\left(A_i U_1V_1^T\right)\Tr\left(A_i U_2V_2^T\right) - \Tr\left(U_1V_1^TV_2U_2^T\right)
	\leq \delta\left(\left\|U_1V_1^T\right\|_F^2 + \left\|U_2V_2^T\right\|_F^2 +
		\left\|U_1V_1^T\right\|_F \left\|U_2V_2^T\right\|_F \right) \label{eqn:sensing-RIP-step1}
\end{align}
where $(\zeta_1)$ follows from the fact that $X_1$ and $X_2$ are rank-$k$ matrices and hence $\calA(\cdot)$ satisfies RIP w.r.t.
those matrices and $(\zeta_2)$ follows from the fact that $\Tr\left(U_1V_1^TV_2U_2^T\right) \leq 
\left\|U_1V_1^T\right\|_F \left\|U_2V_2^T\right\|_F$. Note that if we replace $U_1V_1^T$ by $\lambda U_1V_1^T$
and $U_2V_2^T$ by $\frac{1}{\lambda} U_2V_2^T$
in \eqref{eqn:sensing-RIP-step1} for some non-zero $\lambda \in \R$, the LHS of \eqref{eqn:sensing-RIP-step1} does not
change where as the RHS of \eqref{eqn:sensing-RIP-step1} changes. Optimizing the RHS w.r.t. $\lambda$, we obtain
\begin{align*}
  \sum_i \Tr\left(A_i U_1V_1^T\right)\Tr\left(A_i U_2V_2^T\right) - \Tr\left(U_2^TU_1V_1^TV_2\right)
	\leq 3\delta \left\|U_1V_1^T\right\|_F \left\|U_2V_2^T\right\|_F.
\end{align*}
A similar argument proves the other side of the inequality. This proves the lemma.
\end{proof}

\begin{proof}[Proof of Lemma~\ref{lemma:sensing-updateeqn1}]
We first show that the update \eqref{eq:update_k} reduces to:
\begin{align}\label{eqn:sensing-updateeqn}
  \sum_{q=1}^k \left(\sum_{i=1}^s A_i u_p^{(t)} u_q^{(t)^\dag } A_i^\dag \right)\widehat{v}_q^{(t+1)}
	= \sum_{q=1}^k \left(\sum_{i=1}^s A_i u_p^{(t)} u_q^{*^\dag } A_i^\dag \right) v_q^* \;\;\;\;\; \forall \; p\in[k].
\end{align}
  Let $Err(V)\eqdef \sum_i \left(\Tr\left(A_i M\right) - \Tr\left(A_i U^{(t)}V^\dag \right)\right)^2$. Since $\widehat{V}^{(t+1)}$
minimizes $E(V)$, we have $\nabla_{V} E(\widehat{V}^{(t+1)})=0$.
\begin{align*}
  &\nabla_{v_p} Err(\widehat{V}^{(t+1)})=0 \\
  \Rightarrow &\sum_{i=1}^s \left(\sum_{l=1}^k {v_q^{(t)}}^\dag  A_i u_q^{(t)} - \sum_{l=1}^k \sigma_q^*{v_q^*}^\dag  A_i u_q^* \right)
	A_i u_p =0 \\
  \Rightarrow &\sum_{l=1}^k \sum_{i=1}^s A_i u_p \left({v_q^{(t+1)}}^\dag  A_i u_q^{(t)}\right)
	= \sum_{l=1}^k \sum_{i=1}^s A_i u_p \left(\sigma_q^*{v_q^*}^\dag  A_i u_q^*\right) \\
  \Rightarrow &\sum_{l=1}^k \sum_{i=1}^s A_i u_p \left({u_q^{(t)}}^\dag  A_i^\dag  v_q^{(t+1)}\right)
	= \sum_{l=1}^k \sum_{i=1}^s A_i u_p \left({u_q^*}^\dag  A_i^\dag  \sigma_q^*v_q^*\right) \\
  \Rightarrow &\sum_{l=1}^k \left(\sum_{i=1}^s A_i u_p {u_q^{(t)}}^\dag  A_i^\dag \right) v_q^{(t+1)}
	= \sum_{l=1}^k \left(\sum_{i=1}^s A_i u_p {u_q^*}^\dag  A_i^\dag \right) \sigma_q^*v_q^*
\end{align*}
Define $$S=\left[\begin{matrix}\sigma_1^*I_{n}&\dots&0_{n}\\\vdots&\vdots&\vdots\\0_n&\dots&\sigma_k^*I_n\end{matrix}\right],\ \ v^*=\left[\begin{matrix}v_1^*\\\vdots\\v_k^*\end{matrix}\right],\mbox{ and }\ \  \widehat{v}_1^{(t+1)}=\left[\begin{matrix}\widehat{v}_1^{(t+1)}\\\vdots\\\widehat{v}_k^{(t+1)}\end{matrix}\right].$$ 
Then, 
\begin{align*}
   \widehat{v}_1^{(t+1)}&= B^{-1}CSv^* \\
	&= DSv^* - B^{-1}\left(BD-C\right)Sv^*
\end{align*}
where inverting $B$ is valid since the minimum singular value of $B$ is strictly positive (please refer Lemma
\ref{lemma:sensing-minsingval-B}). Considering the $p^{\textrm{th}}$ block of $\widehat{v}^{(t)}$, we obtain
\begin{align*}
  \widehat{v}_p^{(t+1)} &= \left(\sum_q \iprod{u_p^{(t)}}{u_q^*} \sigma_q^*v_q^*\right) - \left(B^{-1}\left(BD-C\right)Sv^*\right)_p \\
	&= \left(\sum_q \sigma_q^*v_q^* {u_q^*}^\dag  \right)u_p^{(t)} - \left(B^{-1}\left(BD-C\right)Sv^*\right)_p.
\end{align*}
This gives us the following equation for $\widehat{V}^{(t)}$:
\begin{align*}
  \widehat{V}^{(t+1)} = V^*\Sigma^*{U^*}^\dag  U^{(t)} - F
\end{align*}
where $F =\left[ \begin{array}{cccc}
\left(B^{-1}\left(BD-C\right)Sv^*\right)_1 & \left(B^{-1}\left(BD-C\right)Sv^*\right)_2
& \cdots & \left(B^{-1}\left(BD-C\right)Sv^*\right)_k
\end{array} \right].$
 
Hence Proved.
\end{proof}

\subsection{Rank-$1$ Matrix Sensing: Proofs}\label{app:sensing1}
\begin{proof}[Proof of Lemma~\ref{lem:error_1}]
Using definition of the spectral norm: 
\begin{equation}\label{eq:error1_l2}\|B^{-1}\left(\ip{\uo}{\ut}B-C\right)\vo\|\leq \|B^{-1}\|_2\cdot \|\ip{\uo}{\ut}B-C\|_2\cdot \|\vo\|_2.\end{equation}
Consider $B=\sum_i A_i \ut \utt \Ait$. Now, smallest eigenvalue of $B$, i.e., $\lambda_{min}(B)$ is given by: 
\begin{align}
\lambda_{min}(B)&=\min_{\|z\|=1}z^\dag B z=\min_{\|z\|=1} \sum_i z^\dag A_i \ut \utt \Ait z=\min_{\|z\|=1} \sum_i \Tr(A_i \ut z^\dag) \Tr(A_i \ut z^\dag),\nonumber\\
&=\min_{\|z\|=1} \ip{\aff(\ut z^\dag)}{\aff(\ut z^\dag)}\geq 1-3\delta_{2},\label{eq:bmin}
\end{align}
where the last inequality follows using Lemma~\ref{lem:alt_rip}. 
Using \eqref{eq:bmin}, \begin{equation}\|B^{-1}\|_2\leq \frac{1}{1-3\delta_2}.\label{eq:bimax}\end{equation}
\noindent Now, consider $G=\ip{\uo}{\ut}B-C=\sum_i A_i \left(\ip{\uo}{\ut}\ut\utt-\ut\uot\right)\Ait=\sum_i A_i \ut\left(\ip{\uo}{\ut}\ut-\uo\right)^\dag\Ait$. Using definition of the spectral norm: 
\begin{align}
  \|G\|_2&=\max_{\|z\|=1, \|y\|=1} z^\dag G y, \nonumber\\
&=\max_{\|z\|=1, \|y\|=1} \sum_i z^\dag A_i \ut \left(\ip{\uo}{\ut}\ut-\uo\right)^\dag\Ait y, \nonumber\\
&=\max_{\|z\|=1, \|y\|=1} \ip{\aff(\ut z^\dag)}{\aff\left( \left(\ip{\uo}{\ut}\ut-\uo\right) y^\dag\right)}, \nonumber\\
&\leq 3\delta_2\sqrt{1-\ip{\ut}{\uo}^2},\label{eq:gmax}
\end{align}
where the last inequality follows by using Lemma~\ref{lem:alt_rip} and the fact that $\ip{\ut}{\left(\ip{\uo}{\ut}\ut-\uo\right)}=0$. 

Lemma now follows using \eqref{eq:error1_l2}, \eqref{eq:bimax}, \eqref{eq:gmax}. 
\end{proof}

\subsection{Rank-$k$ Matrix Sensing}\label{app:sensingk}
\begin{proof}[Proof of Lemma \ref{lemma:QR-irrelevance}]
Since $\Uht$ and $\widetilde{U}^t$ have full rank and span the same subspace, there exists a $k\times k$, full rank matrix $R$ such that $\Uht = \widetilde{U}^t R = U^t R_U^t R$.
  We have:
  \begin{align*}
    \twonorm{\aff\left(\Uht V^{\dag}\right)-b} &= \twonorm{\aff\left(U^t \left(V\left(R_U^tR\right)^{\dag}\right)^{\dag}\right)-b}
    \geq \twonorm{\aff\left(U^t \left(\widetilde{V}^{t+1}\right)^{\dag}\right)-b}
  \end{align*}
with equality holding in the last step for $V=\widetilde{V}^{t+1}\left(\left(R_U^t R\right)^{\dag}\right)^{-1}$.
The proof of Theorem \ref{thm:sensing} shows that $\widetilde{V}^{t+1}$ is unique and has full rank (since $\dist\left(\widetilde{V}^{t+1},V^*\right)<1$).
This means that $\Vht$ is also unique and is equal to $\widetilde{V}^{t+1}\left(\left(R_U^t R\right)^{\dag}\right)^{-1}$.
This shows that $Span\left(\Vht\right) = Span\left(\widetilde{V}^{t+1}\right)$ and that both $\Vht$ and $\widetilde{V}^{t+1}$ have full rank.
\end{proof}

\begin{lemma}\label{lemma:sensing-minsingval-B}
Let linear measurement $\aff$ satisfy RIP for all $2k$-rank matrices and let $b=\aff(M)$ with $M\in \rmn$ being a rank-$k$ matrix.  Let $\delta_{2k}$ be the RIP constant for rank $2k$-matrices. Then, we have the following bound on the minimum singular value of $B$:
\begin{align}\label{eqn:sensing-minsingval-B}
\sigma_{\textrm{min}}(B) \geq 1-\delta_{2k}.
\end{align}
\end{lemma}
\begin{proof}
  Select any $w\in \R^{nk}$ such that $\twonorm{w}=1$. Let
\begin{align*}
w = \left[ \begin{array}{c}
            w_1 \\
	    w_2 \\
	    \vdots \\
	    w_k
           \end{array}
    \right]
\end{align*}
where each $w_p\in \rn$. Also denote $W \eqdef \left[w_1 w_2 \cdots w_k\right] \in \R^{n\times k}$, i.e., $w=vec(W)$. 

We have,
\begin{align*}
  w^\dag  B w &= \sum_{p,q=1}^k w_p^\dag  B_{pq} w_q
	= \sum_{p,q=1}^k w_p^\dag  \left(\sum_{i=1}^d A_i u_p^t (u_q^t)^\dag  A_i^\dag  \right)w_q 
	= \sum_{i=1}^d \sum_{p,q=1}^k w_p^\dag  A_i u_p^t (u_q^t)^\dag  A_i^\dag  w_q \\
	&= \sum_{i=1}^d \left(\sum_{p=1}^k w_p^\dag  A_i u_p^t\right) \left(\sum_{q=1}^k w_q^\dag  A_i u_q^t\right)=\sum_{i=1}^d \Tr\left(A_iU^tW^\dag \right)^2.
\end{align*}
Now, using RIP (see Definition~\ref{defn:rip}) along with the above equation, we get:
\begin{align*}
	w^\dag  B w&=\sum_{i=1}^d \Tr\left(A_iU^tW^\dag \right)^2\geq \left(1-\delta\right) \frob{U^tW^\dag }^2=\left(1-\delta_{2k}\right) \|W\|_F^2=(1-\delta_{2k})\|w\|^2=(1-\delta_{2k}). 
\end{align*}
Since $w$ was arbitrary, this proves the lemma.
\end{proof}

\begin{proof}[Proof of Lemma~\ref{lemma:sensing-twonorm-F}]
Note that,
\begin{align}
  \twonorm{F(\Sigma^*)^{-1}} \leq \frob{F(\Sigma^*)^{-1}} &= \twonorm{B^{-1} \left(BD-C\right)v^*} \nonumber\\
	&\leq \twonorm{B^{-1}}\twonorm{(BD-C)} \twonorm{v^*} \nonumber\\
	&\leq \frac{\sqrt{k}}{1-\delta_{2k}}\twonorm{(BD-C)} \label{eqn:sensing-twonorm-F-1}
\end{align}
where the last step follows from Lemma~\ref{lemma:sensing-minsingval-B}. 
Now we need to bound $\twonorm{(BD-C)}$. Choose any $w,z \in \R^{nk}$ such that $\twonorm{w}=\twonorm{z}=1$.
As in Lemma \ref{lemma:sensing-minsingval-B}, define the following components of $w$ and $z$:
\begin{align*}
w = \left[ \begin{array}{c}
            w_1 \\
	    w_2 \\
	    \vdots \\
	    w_k
           \end{array}
    \right] \mbox{ and }
z = \left[ \begin{array}{c}
            z_1 \\
	    z_2 \\
	    \vdots \\
	    z_k
           \end{array}
    \right]
\end{align*}
where each $w_p,z_p\in \R^n$ and $W \eqdef \left[w_1 w_2 \cdots w_k\right]$ and $Z \eqdef \left[z_1 z_2
\cdots z_k\right] \in \R^{n\times k}$. We have,
\begin{align*}
  w^\dag \left(BD-C\right)z = \sum_{p,q=1}^k w_p^\dag  \left(BD-C\right)_{pq} z_q
\end{align*}
We calculate $\left(BD-C\right)_{pq}$ as follows:
\begin{align*}
  \left(BD-C\right)_{pq} &= \sum_{l=1}^k B_{pl} D_{lq} - C_{pq} = \left(\sum_{l=1}^k B_{pl} \iprod{u_l^t}{u_q^*}\mathbb{I}_{n\times n}\right) - C_{pq} 
	= \left(\sum_{l=1}^k {u_q^*}^\dag u_l^t \sum_{i=1}^d A_iu_p^t(u_l^t)^\dag A_i^\dag \right) - C_{pq} \\
	&= \sum_{i=1}^d A_iu_p^t {u_q^*}^\dag  \sum_{l=1}^k u_l^t(u_l^t)^\dag  A_i^\dag  -  \sum_{i=1}^d A_iu_p^t(u_q^*)^\dag A_i^\dag  
	= \sum_{i=1}^d A_iu_p^t {u_q^*}^\dag  \left(U^t(U^t)^\dag -\mathbb{I}_{n\times n}\right) A_i^\dag.
\end{align*}
So we have,
\begin{align*}
  w^\dag \left(BD-C\right)z &= \sum_{p,q=1}^k
    w_p^\dag  \sum_{i=1}^d A_iu_p^t {u_q^*}^\dag  \left(U^t(U^t)^\dag -\mathbb{I}_{n\times n}\right) A_i^\dag  z_q = \sum_{i=1}^d \sum_{p,q=1}^k w_p^\dag  A_iu_p^t {u_q^*}^\dag  \left(U^t(U^t)^\dag -\mathbb{I}_{n\times n}\right) A_i^\dag  z_q \\
  &= \sum_{i=1}^d \Tr\left(A_iU^tW^\dag \right) \Tr\left(A_i\left(U^t(U^t)^\dag -\mathbb{I}_{n\times n}\right) U^* Z^\dag \right) \\
  &\stackrel{(\zeta_1)}{\leq} \Tr\left({U^*}^\dag \left(U^t(U^t)^\dag -\mathbb{I}_{n\times n}\right)U^tW^\dag Z\right)
	+ \delta_{2k} \frob{U^tW^\dag } \frob{\left(U^t(U^t)^\dag -\mathbb{I}_{n\times n}\right)U^*Z^\dag } \\
  &\stackrel{(\zeta_2)}{\leq} \delta_{2k} \|W\|_F \sqrt{\frob{(U^*)^\dag \left(U^t(U^t)^\dag -\mathbb{I}_{n\times n}\right)^2U^*}\frob{Z^\dag Z}}\\ 
&\stackrel{(\zeta_3)}{\leq} \delta_{2k}\sqrt{k}\cdot dist(U^t, U^*),
\end{align*}
where $(\zeta_1)$ follows from the fact that $\calA$ satisfies $2k$-RIP and Lemma \ref{lem:alt_rip}, $(\zeta_2)$ follows from the fact that
$\left(U^t(U^t)^\dag -\mathbb{I}_{n\times n}\right)U^t=0$, $(\zeta_3)$ follows from the following: $\|W\|_F=\|w\|_2=1$, $\|Z^\dag Z\|_F\leq \|Z\|_F^2=1$ and 
and finally : $\frob{\left(U^t(U^t)^\dag -\mathbb{I}_{n\times n}\right)U^*}\leq \sqrt{k}\twonorm{\left(U^t(U^t)^\dag -\mathbb{I}_{n\times n}\right)U^*}$. 

Since $w$ and $z$ were arbitrary unit vectors, we can conclude that $\twonorm{BD-C} \leq \delta_{2k} \sqrt{k}\cdot dist(U^t, U^*)$. 
Plugging this bound in \eqref{eqn:sensing-twonorm-F-1} proves the lemma.
\end{proof}

\begin{proof}[Proof of Lemma~\ref{lemma:sensing-minsingval-Rt}]
Note that $\|\Sigma^*(R^{(t+1)})^{-1}\|_2\leq \frac{\sigma^*_1}{\sigma_{\text{min}}(R^{(t+1)})}$. Now,
\begin{align}
  \sigma_{\text{min}}(R^{(t+1)})&=\min_{z, \|z\|_2=1}\|R^{(t+1)}z\|_2=\min_{z, \|z\|_2=1}\|V^{(t+1)}R^{(t+1)}z\|_2,\nonumber\\
&=\min_{z, \|z\|_2=1}\|V^*\Sigma^*{U^*}^\dag  U^{(t)}z-Fz\|_2, \nonumber\\
&\geq \min_{z, \|z\|_2=1}\|V^*\Sigma^*{U^*}^\dag  U^{(t)}z\|_2-\|Fz\|_2,\nonumber\\
&\geq \min_{z, \|z\|_2=1}\|V^*\Sigma^*{U^*}^\dag  U^{(t)}z\|_2-\|F\|_2,\nonumber\\
&\geq \sigma_{k}^*\sigma_{\text{min}}({U^*}^\dag  U^{(t)})-\|F\|_2,\nonumber\\
&\geq \sigma_{k}^*\sqrt{1-\twonorm{{U_{\perp}^*}^\dag U^{(t)}}^2}-\sigma^*_{1}\|F(\Sigma^*)^{-1}\|_2,\nonumber\\
&=\sigma_{k}^*\sqrt{1-\dist(U^*, U^{(t)})^2}-\sigma^*_{1}\|F(\Sigma^*)^{-1}\|_2. 
\end{align}
Lemma now follows using above inequality with Lemma~\ref{lemma:sensing-twonorm-F}. 

\end{proof}
\subsection{Noisy Matrix Sensing: Proofs}\label{app:noisy_sensing}
We now consider an extension of the matrix sensing problem where measurements can be corrupted arbitrarily using a bounded noise. That is, we 
observe $b=\aff\left(M+N\right)$, where $N$ is the noise matrix. 
For this noisy case as well, we show that \as recovers $M$ upto an additive approximation depending on the Frobenius norm of  $N$. 
\begin{theorem}\label{thm:sensing-noisycase}
 Let $M$ and $\calA(\cdot)$ be as defined in Theorem~\ref{thm:sensing_err}. Suppose, \as algorithm (Algorithm \ref{algo:sensing-altmin}) is supplied inputs $\aff$, $b=\aff(M+N)$, where $N$ is the noise matrix s.t. $\left\|N\right\|_F  < \frac{1}{100}\so_k$.  Then, after $T=4\log(2/\epsilon)$ steps, iterates $\widehat{U}^{T}$, $\widehat{V}^{T}$ of \as  satisfy: 
\begin{align*}
  \dist\left(\Vw^{T},V^*\right) \leq \frac{10\|N\|_F}{\so_k} + \epsilon, \ 
  \dist\left(\Uw^{T},U^*\right) \leq \frac{10\|N\|_F}{\so_k} + \epsilon.
\end{align*}
See Definition~\ref{defn:dist} for definition of $\dist\left(U, W\right)$. 
\end{theorem}
\begin{proof}
At a high level, our proof for noisy case follows closely, the exact case proof given in Section~\ref{sec:sensing}. That is, we show that the update of \as algorithm is similar to power-method type update but with two errors terms: one due to incomplete measurements and another due to the noise matrix.

Similar to our proof for sensing problem (Section~\ref{sec:sensing}), we analyze QR-decomposition based updates. That is,
\begin{align*}
\widehat{U}^{t}&= U^{t}R_U^{t}\ \ \text{ (QR decomposition), }\\
\widehat{V}^{t+1}&=\argmin_{V} \  \|\aff(U^t V^\dag )-b\|_2^2,\\ 
\widehat{V}^{t+1}&=V^{t+1}R_V^{t+1}.\ \ \text{ (QR decomposition) }
\end{align*}
Similar to Lemma~\ref{lemma:sensing-updateeqn1}, we can re-write the above given update equation as:
\begin{align}
  \Vht&=\Vo\So(\Uo)^\dag\Ut-F+\Vn \Sn (\Un)^{\dag}U^{(t)}-G,\nonumber\\
  \Vt&=\Vht(\Rt)^{-1},\label{eq:noisy_update_k}
\end{align}
where, $F$ is the error matrix and is as defined in \eqref{eq:F_k} and $G$ is the error matrix due to noise and is given by:
\begin{equation}\label{eq:G_k}
G \eqdef \left[ 
\left(B^{-1}\left(BD^N-C^N\right)S^Nv^N\right)_1\ \ \cdots\ 
 \left(B^{-1}\left(BD^N-C^N\right)S^Nv^N\right)_k
 \right],
\end{equation}
where $B$, $C$ and $D$ defined in the previous section (See \eqref{eq:BCD_k}) and $C^N$ and $D^N$ are defined below:
\begin{align*}
 C^N \eqdef \left[\begin{array}{ccc}
                 C_{11}^N & \cdots & C_{1m}^N \\
		 \vdots &  \ddots &\vdots \\
                 C_{k1}^N & \cdots & C_{km}^N \\
                \end{array} \right]
\mbox{ , }
 D^N \eqdef \left[\begin{array}{ccc}
                 D_{11}^N & \cdots & D_{1m}^N \\
		 \vdots &  \ddots &\vdots \\
                 D_{k1}^N &  \cdots & D_{km}^N \\
                \end{array} \right],
\end{align*}
with $C_{pq}^N \eqdef \sum_{i=1}^d A_i u_p^{(t)}{(u^N_q)}^\dag  A_i^\dag $ and $D_{pq}^N \eqdef \iprod{u_p^{(t)}}{u^N_q}\mathbb{I}_{n\times n}$.
Also, $$S^N=\left[\begin{matrix}\sn_1I_{n}&\dots&0_{n}\\\vdots&\vdots&\vdots\\0_n&\dots&\sn_NI_n\end{matrix}\right],\ \  v^N=\left[\begin{matrix}v^N_1\\\vdots\\v_k^N\end{matrix}\right].$$
Now, multiplying \eqref{eq:noisy_update_k} with $\Vo_\perp$, we get:
\begin{equation*}
{V_{\perp}^*}^\dag \Vt =  ({V_{\perp}^*}^\dag \Vn \Sn U^{N{\dag}}U^{(t)}- {V_{\perp}^*}^\dag  F - {V_{\perp}^*}^\dag  G) {\Rt}^{-1}.
\end{equation*}
That is,
\begin{align}
&\hspace*{-15pt}\dist(\Vo, \Vt)=\|{V_{\perp}^*}^\dag \Vt\|_2\nonumber\\&\hspace*{-15pt}\leq (\| \Vn \Sn (\Un)^{\dag}U^{(t)}\|_2+ \|F |_2 + \| G\|_2) \|{(\Rt)}^{-1}\|_2,\nonumber\\
&\hspace*{-15pt}\leq \left(\sn_1+ \|F(\So)^{-1}\|_2\|\So\|_2 + \| G\|_2\right) \|({\Rt})^{-1}\|_2, \nonumber\\
&\hspace*{-15pt}\leq \left(\sn_1+\frac{\sigma_1^*\delta_{2k}k}{1-\delta_{2k}}\dist(\Ut, \Uo) + \| G\|_2\right) \|({\Rt})^{-1}\|_2,\label{eq:dist_noisy_1}
\end{align}
where the last inequality follows using Lemma~\ref{lemma:sensing-twonorm-F}. 

Now, we break down the proof in the following two steps:
\begin{itemize}
\item Bound $\twonorm{G}$ (Lemma \ref{lemma:sensing-twonorm-G-noisycase}, analogous to Lemma \ref{lemma:sensing-twonorm-F})
\item Bound $\|({\Rt})^{-1}\|_2$ (Lemma \ref{lemma:sensing-minsingval-Rt-noisycase}, similar to Lemma \ref{lemma:sensing-minsingval-Rt})
\end{itemize}
Later in this section, we provide the above mentioned lemmas and their detailed proof. 

Now, by assumption, $\sn_1\leq \frob{N}\leq \so_k$. Also, as $\delta_{2k}\leq 1/2$, $\frac{1}{1-\delta_{2k}}\leq 2$. Finally, assume $\dist(\Vo, \Vt)\geq \max(10\cdot\frac{\sn_1}{\so_k}, 10\frac{\|N\|_F}{\sigma_1^*})$. Using these observations and lemmas~ \ref{lemma:sensing-twonorm-G-noisycase}, \ref{lemma:sensing-minsingval-Rt-noisycase} along with \eqref{eq:dist_noisy_1}, we get:
%
\begin{align}
  \dist(\Vo, \Vt)\leq \frac{0.5\dist(\Uo, \Ut)}{\sqrt{1-\dist(\Ut,\Uo)^2}-0.5\dist(\Uo, \Ut)}. \label{eq:dist_noisy_f}
\end{align}
As, $U^{0}$ is obtained using SVD of $\sum_i A_i b_i$. Hence, using Lemma~\ref{lem:jmd}, we have: 
\begin{align*}
 &\|\aff(U^{0}\Sigma^{0}V^{0}-U^*\Sigma^*(V^*)^\dag)\|_2^2\leq 0.5 \|\aff(N)\|_{2}^2+4\delta_{2k}\|\aff(U^*\Sigma^*(V^*)^\dag)\|_2^2,\\
&\Rightarrow \|U^{0}\Sigma^{0}V^{0}-U^*\Sigma^*(V^*)^\dag\|_F^2\leq \|N\|_{F}^2+4\delta_{2k}(1+\delta_{2k})\|\Sigma^*\|_F^2,\\
&\Rightarrow(\sigma_k^*)^2\|(U^{0}(U^{0})^\dag-I)U^*\|_F^2\leq \|N\|_{F}^2+4\delta_{2k}(1+\delta_{2k})k(\sigma_1^*)^2,\\
&\Rightarrow\dist(U^{0}, U^*)\leq \|(U^{0}(U^{0})^\dag-I)U^*\|_F^2\leq \frac{\|N\|_{F}^2}{(\so_k)^2}+6\delta_{2k}k\left(\frac{\sigma_1^*}{\sigma_k^*}\right)^2<\frac{1}{2},
\end{align*}
where last inequality follows using $\frac{\|N\|_{F}}{\so_k}<1/100$. 

Theorem now follows using above equation with \eqref{eq:dist_noisy_f}. 
\end{proof}
\begin{lemma}\label{lemma:sensing-twonorm-G-noisycase}
Let linear measurement $\aff$ satisfy RIP for all  $2k$-rank matrices and let $b=\aff(M+N)$ with $M\in \rmn$ being a rank-$k$ matrix and let $N=\Un\Sn (\Vn)^\dag$.  Let $\delta_{2k}$ be the RIP constant for rank $2k$-matrices. Then, we have the following bound on $\twonorm{G}$:
\begin{align}\label{eqn:sensing-twonorm-G-noisycase}
    \twonorm{G} \leq \frac{\delta_{2k} \frob{N}}{1-\delta_{2k}}.
\end{align}
\end{lemma}
\begin{proof}
Note that, 
\begin{align}
  \|G\|_2\leq \|G\|_F&=\|B^{-1}(BD^N-C^N)S^Nv^N\|_2\nonumber\\&\leq \|B^{-1}\|_2\|(BD^N-C^N)S^N\|_2\|S^Nv^N\|_2\leq \frac{\sqrt{k}}{1-\delta_{2k}} \|(BD^N-C^N)S^N\|_2,
\end{align}
where the last inequality follows using Lemma~\ref{lemma:sensing-minsingval-B} and the fact that $\|\Vn\|_F=\sqrt{k}$. 
Now let $w=[w_1^\dag\ w_2^\dag\ \dots\ w_k^\dag]^\dag\in \mathbb{R}^{nk}$ and $z=[z_1^\dag\ z_2^\dag\ \dots\ z_n^\dag]^\dag\in \mathbb{R}^{n^2}$ be any two arbitrary vectors such that $\twonorm{w}=\twonorm{z}=1$. Then, 
\begin{align*}
  w^\dag \left(BD^N-C^N\right)S^Nz &= \sum_{p=1}^k \sum_{q=1}^n
    w_p^\dag  \sum_{i=1}^d A_i\ut_p {\un_q}^\dag  \left(\Ut(\Ut)^\dag -\mathbb{I}_{n\times n}\right) A_i^\dag  \sn_qz_q \\
  &= \sum_{i=1}^d \sum_{p=1}^k \sum_{q=1}^n w_p^\dag  A_i\ut_p {\un_q}^\dag  \left(\Ut(\Ut)^\dag -\mathbb{I}_{n\times n}\right) A_i^\dag  \sn_qz_q \\
  &= \sum_{i=1}^d \left(\sum_{p=1}^k w_p^\dag  A_i\ut_p \right) \left(\sum_{q=1}^n \sn_qz_q^\dag  A_i \left(\Ut(\Ut)^\dag -\mathbb{I}_{n\times n}\right)\un_q \right) \\
  &= \sum_{q=1}^n \left(\sum_{i=1}^d \Tr\left(A_iUW^\dag \right) \Tr\left(A_i\left(\Ut(\Ut)^\dag -\mathbb{I}_{n\times n}\right) \un_q \sn_qz_q^\dag \right)\right).
\end{align*}
Now, using RIP, we get:
\begin{align*}
  w^\dag \left(BD^N-C^N\right)z&\leq \sum_{q=1}^n {\un_q}^\dag \left(\Ut{\Ut}^\dag -\mathbb{I}_{n\times n}\right)\Ut W^\dag \sn_qz_q
	+ \delta_{2k} \frob{\Ut W^\dag } \frob{\left(\Ut{\Ut}^\dag -\mathbb{I}_{n\times n}\right)\un_q\sn_q{z_q}^\dag } \\
  &\leq \sum_{q=1}^n \delta_{2k} \|W^\dag\|_F \|(\Ut(\Ut)^\dag -\mathbb{I}_{n\times n})\un_q\|_2\|\sn_q{z_q}\|_2,\\
  &\leq \delta_{2k} \ \sum_{q=1}^n \twonorm{\un_q} \twonorm{\sn_qz_q}= \delta_{2k} \sum_{q=1}^n \sn_q\twonorm{z_q},\\&\leq \delta_{2k} \sqrt{\sum_{q=1}^n \twonorm{z_q}^2} \sqrt{\sum_{q=1}^n (\sn_q)^2}\leq \delta_{2k} \frob{N}.
\end{align*}
This finishes the proof.
\end{proof}
\begin{lemma}\label{lemma:sensing-minsingval-Rt-noisycase}
Assuming conditions of Lemma~\ref{lemma:sensing-twonorm-G-noisycase}, we have the following bound on the minimum singular value of $R^{(t)}$:
\begin{align*}
  \sigma_{\textrm{min}}\left(\Rt\right) \geq \so_k\sqrt{1-\dist(\Ut, \Uo)^2} - \sn_1 - \twonorm{F} - \twonorm{G}.
\end{align*}
\end{lemma}
\begin{proof}
  Similar to the proof of Lemma \ref{lemma:sensing-minsingval-Rt}, we have the following set of inequalities:
\begin{align*}
  \sigma_{\textrm{min}}\left(\Rt\right) = \min_{\twonorm{z}=1} \twonorm{\Rt z}
	&= \min_{\twonorm{z}=1} \twonorm{V^{(t)}\Rt z} \\
	&= \min_{\twonorm{z}=1} \twonorm{V^*\So{U^*}^\dag \Ut z + \Vn\Sn(\Un)^{\dag}U^{(t)} z - Fz - Gz} \\
	&\geq \min_{\twonorm{z}=1} \twonorm{V^*\So{U^*}^\dag \Ut z} - \twonorm{\Vn\Sn(\Un)^{\dag}} - \twonorm{F} - \twonorm{G} \\
	&\geq \so_k\min_{\twonorm{z}=1} \twonorm{{U^*}^\dag \Ut z} - \sn_1 - \twonorm{F} - \twonorm{G} \\
	&\geq \so_k\sqrt{1-\twonorm{{U_{\perp}^*}^\dag U}^2} - \sn_1 - \twonorm{F} - \twonorm{G} \\
	&= \so_k\sqrt{1-\dist(\Ut,\Uo)^2} - \sn_1 - \twonorm{F} - \twonorm{G}.
\end{align*}
This proves the lemma.
\end{proof}
\subsection{Stagewise Alternating Minimization for Matrix Sensing: Proofs}\label{app:stagewise_sensing}
\begin{proof}[Proof of Lemma~\ref{lemma:sam_svp_err}]
As the initial point of the $i$-th stage is obtained by one step of SVP \cite{JainMD10}, using Lemma~\ref{lem:jmd}, we obtain:
$$\left\|M - \widehat{U}^0_{1:i}(\widehat{V}_{1:i}^0)^{\dag}\right\|_F^2 \leq \sum_{j=i+1}^k(\so_j)^2+2\delta_{2k}\|M-\Uw^T_{1:i-1}V^T_{1:i-1}\|_F^2.$$
Now, by assumption over the $(i-1)$-th stage error (this assumption follows from the inductive hypothesis in proof of Theorem~\ref{thm:sam_err}), 
$$\left\|M - \widehat{U}^0_{1:i}(\widehat{V}_{1:i}^0)^{\dag}\right\|_F^2 \leq \sum_{j=i+1}^k(\so_j)^2+2\delta_{2k} 16k(\so_i)^2.$$
Lemma now follows by setting $\delta_{2k}\leq \frac{1}{3200k}$. 
\end{proof}
\begin{proof}[Proof of Lemma~\ref{lemma:sam_altmin}]
For our proof, we consider two cases: a) $\frac{\so_i}{\so_{i+1}} < 5\sqrt{k}$, b) $\frac{\so_i}{\so_{i+1}} \geq 5\sqrt{k}$.\\
{\bf Case (a)}: In this case, using monotonicity of the AltMin algorithm directly gives error bound. That is, $\|M - \widehat{U}^T_{1:i}(\widehat{V}_{1:i}^T)^{\dag}\|_F^2\leq \|M-\Uw^0_{1:i}V^0_{1:i}\|_F^2 \leq k(\so_{i+1})^2+\frac{25k}{100}(\so_{i+1})^2.$\\
{\bf Case (b)}: At a high level,  if $\frac{\so_i}{\so_{i+1}} \geq 5\sqrt{k}$ then $U^0_{1:i}$ is ``close'' to $U^*_{1:i}$ and hence the error bound follows by using an analysis similar to the noisy case. Note that $\so_{i+1}$ being small implies that the ``noise'' is small. See Lemma~\ref{lemma:cond-num-bad-altmin-performance} for a formal proof of this case. \hfill 
\end{proof}

\begin{lemma}\label{lemma:cond-num-bad-altmin-performance}
Assume conditions given in Theorem~\ref{thm:sam_err} are satisfied and  let $\frac{\so_i}{\so_{i+1}} \geq 5\sqrt{k}$. Also, let $$\left\|M - \widehat{U}^0_{1:i}(\widehat{V}^0_{1:i})^{\dag}\right\|_F^2 \leq \sum_{j=i+1}^k (\so_j)^2 + \frac{1}{100} (\so_i)^2.$$ Then, $U^T_{1:i}$, $V^T_{1:i}$ satisfy:$$\|M-\Uw^T_{1:i}V^T_{1:i}\|_F^2\leq \max(\epsilon, 16k(\so_{i+1})^2),$$
\end{lemma}
\begin{proof}
We first show that if $\sigma_i$ and $\sigma_{i+1}$ have large gap then $\forall\; t$, the $t^{\textrm{th}}$ iterate of the $i$-th stage, $\Uw^t_{1:i}$ is close to $\Uo_{1:i}$. Let $\Up^t$ be a basis of the subspace orthogonal to $\Uw^t_{1:i}$. 
\begin{align}
  \|\left(\Up^t\right)^\dag(M-\Uw^t_{1:i}(\Vw^t_{1:i})^\dag)\|_2&=\|\left(\Up^t\right)^\dag M\|_2\geq \|\left(\Up^t\right)^\dag\Uo_{1:i}\So_{1:i}(\Vo_{1:i})^\dag\|_2-\|\left(\Up^t\right)^\dag\Uo_{i+1:k}\So_{i+1:k}(\Vo_{i+1:k})^\dag\|_2,\nonumber\\
&\geq \so_i\|\left(\Up^t\right)^\dag\Uo_{1:i}\|_2-\so_{i+1}\geq \so_i(\|\left(\Up^t\right)^\dag\Uo_{1:i}\|_2-\frac{1}{5\sqrt{k}}). \label{eq:sam_bc_uperp1}
\end{align}
We also have:
\begin{align}
\|\left(\Up^t\right)^\dag(M-\Uw^t_{1:i}(\Vw^t_{1:i})^\dag)\|_2^2 &
	\leq \|M-\Uw^t_{1:i}(\Vw^t_{1:i})^\dag\|_F^2\leq \frac{1}{1-\delta_{2k}}\twonorm{\aff\left(M-\Uw^t_{1:i}(\Vw^t_{1:i})^\dag\right)}^2 \nonumber\\
	&\stackrel{(\zeta_1)}{\leq} \frac{1}{1-\delta_{2k}}\twonorm{\aff\left(M-\Uw^0_{1:i}(\Vw^0_{1:i})^\dag\right)}^2 \leq \frac{1+\delta_{2k}}{1-\delta_{2k}}\|M-\Uw^0_{1:i}(\Vw^0_{1:i})^\dag\|_F^2 \nonumber\\
	&\leq \frac{1+\delta_{2k}}{1-\delta_{2k}} \left(\sum_{j=i+1}^k \soj + \frac{1}{100} \soi\right) \nonumber\\
	&\leq \frac{1+\delta_{2k}}{1-\delta_{2k}} \left(k\soii+\frac{1}{100}\soi\right), \label{eq:sam_bc_uperp2}
\end{align}
where $(\zeta_1)$ follows from the fact that lines $5-8$ of Algorithm \ref{algo:sensing-altmin-improved} never increases $\twonorm{\aff\left(M-\Uw^t_{1:i}(\Vw^t_{1:i})^\dag\right)}$.
Using \eqref{eq:sam_bc_uperp1}, \eqref{eq:sam_bc_uperp2}, and $\frac{\so_i}{\so_{i+1}} \geq 5\sqrt{k}$, we obtain the following bound:
\begin{equation}
  \label{eq:sam_bc_uperp3}
  \|\left(\Up^t\right)^\dag \Uo_{1:i}\|_2\leq \frac{1}{2} \; \forall \;t.
\end{equation}
Now, we consider the update equation for $\Vht$:
$$\Vht=\argmin_{\Vw}\|\aff(\Uw^t_{1:i}\Vw-\Uo_{1:i}\So_{1:i}(\Vo_{1:i})^\dag-\Uo_{i+1:k}\So_{i+1:k}(\Vo_{i+1:k})^\dag)\|_2^2.$$
Note that, the update is same as noisy case with noise matrix $N=\Uo_{i+1:k}\So_{i+1:k}(\Vo_{i+1:k})^\dag$ from \eqref{eq:noisy_update_k}:
\begin{align}\label{eqn:updateeqn-stagewise}
  \Vht = \Vo_{1:i}\So_{1:i}(\Uo_{1:i})^\dag\Ut_{1:i}-F+ \Vo_{i+1:k}\So_{i+1:k}(\Uo_{i+1:k})^\dag\Ut_{1:i} - G,
\end{align}
where $F$ and $G$ are given by \eqref{eq:F_k}, \eqref{eq:G_k}. 
Multiplying \eqref{eqn:updateeqn-stagewise}
from the left by $V_{\perp}^{\dag}=I-\Vt(\Vt)^\dag$, we obtain:
\begin{align}
  &0 = V_{\perp}^{\dag}\Vht_{1:i} = V_{\perp}^{\dag} \left(\Vo_{1:i}\So_{1:i}(\Uo_{1:i})^\dag\Ut_{1:i} - F + \Vo_{i+1:k}\So_{i+1:k}(\Uo_{i+1:k})^\dag\Ut_{1:i} - G\right) \nonumber\\
\Rightarrow &V_{\perp}^{\dag} \Vo_{1:i}\So_{1:i}(\Uo_{1:i})^\dag\Ut_{1:i} = V_{\perp}^{\dag} \left( F - \Vo_{i+1:k}\So_{i+1:k}(\Uo_{i+1:k})^\dag\Ut_{1:i} + G\right) \nonumber\\
\Rightarrow &\frob{V_{\perp}^{\dag} \Vo_{1:i}\So_{1:i}(\Uo_{1:i})^\dag\Ut_{1:i}} \leq \frob{F} + \frob{V_{\perp}^{\dag}\Vo_{i+1:k}\So_{i+1:k}(\Uo_{i+1:k})^\dag\Ut_{1:i}} + \frob{G} \nonumber\\
\Rightarrow &\frob{V_{\perp}^{\dag} \Vo_{1:i}\So_{1:i}} \leq \frac{1}{\sigma_{\textrm{min}}\left((\Uo_{1:i})^\dag\Ut_{1:i}\right)} \left(\frob{F} + \frob{V_{\perp}^{\dag}\Vo_{i+1:k}\So_{i+1:k}(\Uo_{i+1:k})^\dag\Ut_{1:i}} + \frob{G}\right), \label{eqn:weighted-dist-decay1}
\end{align}
where the last inequality follows using the fact that $\sigma_{\min}(A)\|B\|_F\leq \|AB\|_F$. Using Lemma~\ref{lemma:sensing-twonorm-G-noisycase}, and a modification of Lemma~\ref{lemma:sensing-twonorm-F}, we get:
\begin{align}
  &\frob{F} \leq \delta_{2k} \frob{U_{\perp}^{\dag} \Uo_{1:i}\So_{1:i}},\ \ \ \frob{G} \leq \delta_{2k} \frob{U_{\perp}^{\dag} \Uo_{i+1:k}\So_{i+1:k}} \leq \delta_{2k}\sqrt{k}\sigma_{i+1} \label{eqn:frob-G-bound}.
\end{align}
Using \eqref{eqn:weighted-dist-decay1}, \eqref{eqn:frob-G-bound}, and the fact that $\sigma_{\min}(U_\perp^\dag\Uo_{1:i})=\sqrt{1-\|U_\perp^\dag\Uo_{1:i}\|_2^2}$,
\begin{align*}
  \frob{V_{\perp}^{\dag} \Vo_{1:i}\So_{1:i}} 
	&\leq \frac{2}{\sqrt{3}} \left(\delta_{2k} \frob{U_{\perp}^{\dag} \Uo_{1:i}\So_{1:i}}
		+ \sqrt{\sum_{j=i+1}^k \soj} + \delta_{2k}\sqrt{k}\sigma_{i+1} \right).
\end{align*}
Assuming $\frob{U_{\perp}^{\dag} \Uo_{1:i}\So_{1:i}} > 2 \sqrt{\sum_{j=i+1}^k \sigma_j^2}$, we obtain:
\begin{align}
  \frob{V_{\perp}^{\dag} \Vo_{1:i}\So_{1:i}} 
	&\leq \frac{2}{3} \frob{U_{\perp}^{\dag} \Uo_{1:i}\So_{1:i}}. 
\end{align}
Using similar analysis, we can show that, $$\frob{U_{\perp}^{\dag} \Uo_{1:i}\So_{1:i}} 
	\leq \frac{2}{3} \frob{V_{\perp}^{\dag} \Vo_{1:i}\So_{1:i}}. $$
So after $T\geq 8\log(k\so_i)$ iterations, we have:
\begin{align*}
  \frob{U_{\perp}^{\dag} \Uo_{1:i}\So_{1:i}}^2 \leq 4\sum_{j=i+1}^k \soj. 
\end{align*}
Using the above inequality, we now bound the error after $T\geq  8\log(k\so_i)$ iterations of the $i$-th stage: 
\begin{align}\label{eq:sam_err_bc_t_1}
  \left\|M - \Uw^T_{1:i}(\Vw^T_{1:i})^{\dag}\right\|_F
	\leq \left\| \Uo_{1:i}\So_{1:i}\left(\Vo_{1:i}\right)^{\dag} - \Uw^T_{1:i}(\Vw^T_{1:i})^{\dag}\right\|_F + \frob{\Uo_{i+1:k}\So_{i+1:k}\left(\Vo_{i+1:k}\right)^{\dag}}.
\end{align}

For the first term, we have:
\begin{align}
	&\left\| \Uo_{1:i}\So_{1:i}\left(\Vo_{1:i}\right)^{\dag} - \Uw^T_{1:i}(\Vw^T_{1:i})^{\dag}\right\|_F^2 \nonumber\\
	&= \frob{\Uo_{1:i}\So_{1:i}\left(\Vo_{1:i}\right)^{\dag} - U^T_{1:i}\left(U^T_{1:i}\right)^{\dag}\Uo_{1:i}\So_{1:i}\left(\Vo_{1:i}\right)^{\dag}
		+ U^T_{1:i}\left(U^T_{1:i}\right)^{\dag}\Uo_{1:i}\So_{1:i}\left(\Vo_{1:i}\right)^{\dag} - \Uw^T_{1:i}(\Vw^T_{1:i})^{\dag}}^2 \nonumber\\
	&= \frob{\left(I-U^T_{1:i}\left(U^T_{1:i}\right)^{\dag}\right)\Uo_{1:i}\So_{1:i}\left(\Vo_{1:i}\right)^{\dag} }^2
		+ \frob{U^T_{1:i}\left(U^T_{1:i}\right)^{\dag}\Uo_{1:i}\So_{1:i}\left(\Vo_{1:i}\right)^{\dag} - \Uw^T_{1:i}(\Vw^T_{1:i})^{\dag}}^2 \nonumber\\
	&\stackrel{(\zeta_1)}{\leq} \frob{U_{\perp}^{\dag} \Uo_{1:i}\So_{1:i}}^2 + \frob{U_{1:i}^T\left(F+G-\left(\Ut_{1:i}\right)^{\dag}\Uo_{i+1:k}\So_{i+1:k}(\Vo_{i+1:k})^\dag\right)}^2 \nonumber\\
	&= \frob{U_{\perp}^{\dag} \Uo_{1:i}\So_{1:i}}^2 + \frob{F+G-\left(\Ut_{1:i}\right)^{\dag}\Uo_{i+1:k}\So_{i+1:k}(\Vo_{i+1:k})^\dag}^2 \nonumber\\
	&= \frob{U_{\perp}^{\dag} \Uo_{1:i}\So_{1:i}}^2 + 3\frob{F}^2+3\frob{G}^2+3\frob{\Uo_{i+1:k}\So_{i+1:k}(\Vo_{i+1:k})^\dag}^2 \nonumber\\
	&\stackrel{(\zeta_2)}{\leq} (1+3\delta_{2k}^2) \frob{U_{\perp}^{\dag} \Uo_{1:i}\So_{1:i}}^2 + 3(1+\delta_{2k}^2) \frob{\Uo_{i+1:k}\So_{i+1:k}(\Vo_{i+1:k})^\dag}^2 \nonumber\\
	&\leq 8k (\sigma_{i+1}^*)^2,
\label{eq:sam_err_bc_t_2}
\end{align}
where $(\zeta_1)$ follows from \eqref{eqn:updateeqn-stagewise} and $(\zeta_2)$ follows from \eqref{eqn:frob-G-bound}.
Using \eqref{eq:sam_err_bc_t_1} and \eqref{eq:sam_err_bc_t_2}, we obtain the following bound:
\begin{align}\label{eq:sam_err_bc_t}
  \left\|M - \Uw^T_{1:i}(\Vw^T_{1:i})^{\dag}\right\|_F \leq 4\sqrt{k} \sigma_{i+1}^*.
\end{align}
Hence Proved.
\end{proof}
\section{Matrix Completion: Proofs}\label{app:mc}
\begin{proof}[Proof Of Theorem~\ref{thm:sampling-altmin-main}]
Using Theorem~\ref{thm:sampling-altmin-geomconv}, after $O(\log (1/\epsilon))$ iterations, we get: $$\dist(\Ut, \Uo)\leq \epsilon,\ \dist(\Vt, \Vo)\leq \epsilon.$$
Now, using  \eqref{eq:updatecomp_k}, the residual after $t$-th step is given by:
$$M-\Ut(\Vht)^\dag=(I-\Ut(\Ut)^\dag)M-\Ut F^\dag.$$
That is, 
$$\|M-\Ut(\Vht)^\dag\|_F\leq \|(I-\Ut(\Ut)^\dag)M\|_F+\|F\|_F\leq \sqrt{k}\|(I-\Ut(\Ut)^\dag)U^*\So\|_2+\|F\|_F\leq \sqrt{k}\so_1\dist(\widehat{U}^t, \Uo).$$
Now, using the fact that $\dist(\Ut, \Uo)\leq \epsilon$ and the above equation, we get: 
$$\|M-\Ut(\Vht)^\dag\|_F\leq \sqrt{k}\so_1\epsilon+\|F\|_F\stackrel{\zeta_1}{\leq}\sqrt{k}\so_1\epsilon+\so_1\sqrt{k}\epsilon\leq 2\so_1\sqrt{k}\epsilon,$$
where $\zeta_1$ follows by Lemma~\ref{lemma:Fbound_comp} and setting $\delta_{2k}$ appropriately. Theorem \ref{thm:sampling-altmin-main} now follows by setting $\epsilon'=2\sqrt{k}\|M\|_F\epsilon$.
\end{proof}
\subsection{Initialization: Proofs}\label{app:comp_init}
\begin{proof}[Proof Of Lemma~\ref{lemma:svd-clipping-jointguarantee}]
  From Lemma \ref{lemma:first-step-svd}, we see that $U^0$ obtained after step $3$ of Algorithm~\ref{algo:completion-altmin} satisfies: $\dist\left(U^{0},U^*\right)\leq \frac{1}{64 k}$. Lemma now follows by using the above mentioned observation with Lemma~\ref{lemma:clipping-incoherence-distance}.\end{proof}
We now provide the two results used in the above lemma. 
\begin{lemma}\label{lemma:first-step-svd}
After step $3$ in Algorithm \ref{algo:completion-altmin}, whp we have:
  \begin{align*}
    \dist\left(U^{0},U^*\right) \leq \frac{1}{64 k}
  \end{align*}
\end{lemma}
\begin{proof}
  From Theorem 3.1 in \cite{KeshavanOM2009}, we have the following result:
  \begin{align*}
    \twonorm{M - M_k} \leq C \left(\frac{k}{p\sqrt{mn}}\right)^{\frac{1}{2}}\|M\|_F.
  \end{align*}
Let $U^{(0)}\Sigma V^{\dag}$ be the top $k$ singular components of $M_k$.
We also have:
\begin{align*}
  \twonorm{M - M_k}^2 &= \twonorm{U^* \Sigma^* (V^*)^{\dag} - U^{(0)}\Sigma V^{\dag}}^2\\
	&= \left\|U^* \Sigma^* (V^*)^{\dag} - U^{(0)}\left(U^{(0)}\right)^{\dag}U^* \Sigma^* (V^*)^{\dag}
	  + U^{(0)}\left(U^{(0)}\right)^{\dag}U^* \Sigma^* (V^*)^{\dag} - U^{(0)}\Sigma V^{\dag}\right\|_2^2 \\
	&= \left\|\left(I-U^{(0)}\left(U^{(0)}\right)^{\dag}\right)U^* \Sigma^* (V^*)^{\dag}
	  + U^{(0)} \left(\left(U^{(0)}\right)^{\dag}U^* \Sigma^* (V^*)^{\dag} - \Sigma V^{\dag}\right)\right\|_2^2 \\
	&\stackrel{(\zeta_1)}{\geq} \twonorm{\left(I-U^{(0)}\left(U^{(0)}\right)^{\dag}\right)U^* \Sigma^* (V^*)^{\dag} }^2 = \twonorm{\left(U_{\perp}^{(0)}\right)^{\dag}U^* \Sigma^*}^2 \geq \left(\sigma_k^*\right)^2 \twonorm{\left(U_{\perp}^{(0)}\right)^{\dag}U^*}^2,
\end{align*}
where $(\zeta_1)$ follows from the fact that the column space of the first two terms in the equation is $U_{\perp}^{(0)}$ where as the column space
of the last two terms is $U^{(0)}$. Using the above two inequalities, we get:
\begin{align*}
  \twonorm{\left(U_{\perp}^{(0)}\right)^{\dag}U^*} \leq C\cdot\frac{\so_1}{\sigma_k^*} \cdot \frac{k}{\sqrt{mp}}
	\leq \frac{1}{10^4 k},
\end{align*}
if $p > \frac{C'k^4\log n}{m}\cdot \frac{(\so_1)^2}{(\so_k)^2}$ for a large enough constant $C'$.
\end{proof}
\begin{lemma}\label{lemma:clipping-incoherence-distance}(Analysis of step $4$ of Algorithm \ref{algo:completion-altmin})
  Suppose $U^*$ is incoherent with parameter $\mu$ and $U$ is an orthonormal column matrix such that $\dist\left(U,U^*\right) \leq \frac{1}{64k}$.
Let $U^c$ be obtained from $U$ by setting all entries greater than $\frac{2\mu\sqrt{k}}{\sqrt{n}}$  to zero. Let $\widetilde{U}$ be an orthonormal basis of $U^c$. Then,\vspace*{-4pt}\setlength{\topsep}{-5pt}\setlength{\itemsep}{-5pt}
\begin{itemize}\setlength{\topsep}{-5pt}\setlength{\itemsep}{-5pt}
  \item	$\dist\left(\widetilde{U},U^*\right) \leq 1/2$ and
  \item	$\widetilde{U}$ is incoherent with parameter $4\mu \sqrt{k}$.\vspace*{-5pt}
\end{itemize}
\end{lemma}
\begin{proof}
  Since $\dist\left(U,U^*\right) \leq d$, we have that for every $i$, $\exists \breve{u}_i \in \textrm{Span}(U^*), \twonorm{\breve{u}_i}=1$ such that
 $\iprod{u_i}{\breve{u}_i} \geq \sqrt{1-d^2}$. Also, since $\breve{u}_i \in \textrm{Span}(U^*)$, we have that $\breve{u}_i$
 is incoherent with parameter $\mu \sqrt{k}$:
\begin{align*}
  \twonorm{\breve{u}_i}=1 \mbox{ and }
  \infnorm{\breve{u}_i} \leq \frac{\mu \sqrt{k}}{\sqrt{m}}.
\end{align*}
Let $u_i^c$ be the vector obtained by setting all the elements of $u_i$ with magnitude greater than $\frac{2\mu\sqrt{k}}{\sqrt{m}}$ to zero and let
$u_i^{\overline{c}} \eqdef u_i - u_i^c$. 
Now, note that if for element $j$ of $u_i$ we have $\left|u_{i}^j\right| > \frac{2\mu\sqrt{k}}{\sqrt{m}}$, then, $|(u_{i}^c)^j-\breve{u}_i^j|=\left|\breve{u}_i^j\right|\leq \frac{\mu\sqrt{k}}{\sqrt{m}} \leq \left|u_i^j - \breve{u}_i^j\right|$. Hence, 
\begin{align*}
  \twonorm{u_i^c - \breve{u}_i} \stackrel{(\zeta_1)}{\leq} \twonorm{u_i - \breve{u}_i}
	= \left(\twonorm{u_i}^2 + \twonorm{\breve{u}_i}^2 - 2 \iprod{u_i}{\breve{u}_i}\right)^{\frac{1}{2}} \leq \sqrt{2}d,
\end{align*}
This also implies the following:
\begin{align*}
  \twonorm{u_i^c} &\geq \twonorm{\breve{u}_i} - \sqrt{2}d = 1 - \sqrt{2}d\ \ \mbox{, and }\\
  \twonorm{u_i^{\overline{c}}} &\leq \sqrt{1 - \twonorm{u_i^c}^2} \leq \sqrt{2d(\sqrt{2}-d)} \leq 2 \sqrt{d},\ \  \mbox{ for } d < \frac{1}{\sqrt{2}}.
\end{align*}

Let $U^c = \widetilde{U}\Lambda^{-1}$ (QR decomposition). Then, for any $u_{\perp}^* \in \textrm{Span}(U_{\perp}^*)$ we have:
\begin{align*}
  \twonorm{\left(u_{\perp}^*\right)^{\dag} \widetilde{U}} &= \twonorm{\left(u_{\perp}^*\right)^{\dag} U^c \Lambda} \leq \twonorm{\left(u_{\perp}^*\right)^{\dag} U^c} \twonorm{\Lambda} \leq \left(\twonorm{\left(u_{\perp}^*\right)^{\dag} U} + \twonorm{\left(u_{\perp}^*\right)^{\dag} U^{\overline{c}} }\right) \twonorm{\Lambda} \\
	&\leq \left(d + \twonorm{ U^{\overline{c}} }\right)\twonorm{\Lambda}
	\leq \left(d + \frob{ U^{\overline{c}} }\right)\twonorm{\Lambda}
	\leq \left(d + 2\sqrt{kd} \right)\twonorm{\Lambda}
	\leq 3\sqrt{kd} \twonorm{\Lambda}.
\end{align*}
We now bound $\twonorm{\Lambda}$ as follows:
\begin{align*}
  \twonorm{\Lambda}^2 = \frac{1}{\sigma_{\textrm{min}}\left(\Lambda^{-1}\right)^2} 
	= \frac{1}{\sigma_{\textrm{min}}\left(\widetilde{U}\Lambda^{-1}\right)^2} 
	= \frac{1}{\sigma_{\textrm{min}}\left(U^c\right)^2} 
	\leq \frac{1}{1 - \twonorm{U^{\overline{c}}}^2} 
	\leq \frac{1}{1 - 4kd} \leq 4/3,
\end{align*}
where we used the fact that $d < \frac{1}{16k}$.
So we have:
\begin{align*}
  \twonorm{\left(u_{\perp}^*\right)^{\dag} \widetilde{U}} \leq  3 \sqrt{kd} \cdot 4/3 = 4\sqrt{kd}.
\end{align*}
This proves the first part of the lemma.

Incoherence of $\widetilde{U}$ follows using the following set of inequalities: 
\begin{align*}
\mu(\widetilde{U})=\frac{\sqrt{m}}{\sqrt{k}}\max_i \|e_i^\dag \widetilde{U}\|_2 \leq \frac{\sqrt{m}}{\sqrt{k}}\max_i \|e_i^\dag U^c \Lambda\| \leq \frac{\sqrt{m}}{\sqrt{k}}\max_i \|e_i^\dag U^c\|_2\|\Lambda\|_2\leq 4\mu \sqrt{k}. 
\end{align*}
\end{proof}
\subsection{Rank-$1$ Matrix Completion: Proofs}\label{app:rank1_comp}
\begin{proof}[Proof Of Lemma~\ref{lem:errb_comp1}]
Using the definition of spectral norm, 
\begin{align*}
  \|B^{-1}\left(\ip{\uo}{\ut}B-C\right)\vo\|_2\leq \|B^{-1}\|_2\|(\ip{\uo}{\ut}B-C)\vo\|_2.
\end{align*}
As $B$ is a diagonal matrix, $\|B^{-1}\|_2= \frac{1}{\min_iB_{ii}}\leq \frac{1}{1-\delta_2}$, where the last inequality follows using Lemma~\ref{lem:errb_comp2}. The lemma now follows using the above observation and Lemma~\ref{lem:bccompbound}.
\end{proof}
\begin{lemma}\label{lem:errb_comp2}
Let $M=\so\uo(\vo)^\dag$, $p$, $\Omega$, $\ut$ be as defined in Lemma~\ref{lem:errb_comp1}. 
Then, w.p. at least $1-\frac{1}{n^3}$,
$$\left|\frac{\sum_{i:(i,j)\in \Omega} (\ut_i)^2}{p}-1\right| \leq \delta_2,\ \left|\frac{\sum_{i:(i,j)\in \Omega} \ut_i \uo_i}{p}-\ip{\ut}{\uo}\right|\leq \delta_2.$$ 
\end{lemma}
\begin{proof}
Since the first part of the lemma is a direct consequence of the second part, we will prove only the second part.
Let $\delta_{ij}$ be a Bernoulli random variable that indicates membership of index $(i,j)\in \Omega$. That is, $\delta_{ij}=1$ w.p. $p$ and $0$ otherwise. Define $Z_j=\frac{1}{p}\sum_{i} \delta_{ij} \ut_i \uo_i$. Note that $\mathbb{E}[Z_j]=\ip{\ut}{\uo}$. Furthermore, $\mathbb{E}[Z_j^2]=\left(\frac{1}{p}-1\right)\sum_{i}(\ut_i \uo_i)^2\leq \frac{\mu_1^2}{mp}$ and $\max_i |\ut_i \uo_i|\leq \frac{\mu_1^2}{m}$. Using Bernstein's inequality, we get:  
\begin{align}
  \Pr(\left|Z_j-\ip{\ut}{\uo}\right|>\delta_2)\leq \exp\left(-\frac{\delta_2^2mp/2}{\mu_1^2+\mu_1^2\delta_2/3}\right). 
\end{align}
Using union bound (for all $j$) and for $p\geq \frac{9\mu_1^2\log n}{m\delta_2^2}$, w.p. $1-\frac{1}{n^3}$: $\forall j, \ip{\ut}{\uo}-\delta_2\leq Z_j \leq \ip{\ut}{\uo}+\delta_2$.
\end{proof}
\begin{lemma}\label{lem:bccompbound}
Let $M=\so\uo(\vo)^\dag$, $p$, $\Omega$, $\ut$ be as defined in Lemma~\ref{lem:errb_comp1}. 
Then, w.p. at least $1-\frac{1}{n^3}$,
$$\|(\ip{\uo}{\ut}B-C)\vo\|_2\leq \delta_2\sqrt{1-\ip{\uo}{\ut}^2}. $$
\end{lemma}
\begin{proof}
Let $x\in \rn$ be a unit vector. Then, $\forall x$: 
\begin{align}
x^\dag(\ip{\uo}{\ut}B-C)\vo&=\frac{1}{p}\sum_{ij\in \Omega}x_j\vo_j(\ip{\uo}{\ut}(\ut_i)^2-\ut_i\uo_i)\nonumber\\
&\stackrel{(\zeta_1)}{\leq} \frac{1}{p}C\sqrt{mp}\sqrt{\sum_j x_j^2(\vo_j)^2}\sqrt{\sum_i(\ip{\uo}{\ut}(\ut_i)^2-\ut_i\uo_i)^2}, \nonumber\\
&\stackrel{(\zeta_2)}{\leq}\frac{1}{p}C\frac{\sqrt{mp}\mu_1^2}{n} \sqrt{1-\ip{\uo}{\ut}^2},
\end{align}
where $C>0$ is a global constant and $(\zeta_1)$ follows by using a modified version of Lemma 6.1 by \cite{KeshavanOM2009} (see Lemma~\ref{lem:komlem}) and $(\zeta_2)$ follows by using incoherence of $\vo$ and $\ut$. Lemma now follows by observing that $\max_{x, \|x\|_2=1} x^\dag(\ip{\uo}{\ut}B-C)\vo=\|(\ip{\uo}{\ut}B-C)\vo\|_2$ and $p>\frac{C\mu_1^2\log n}{m\delta_2^2}$.
\end{proof}
\begin{proof}[Proof of Lemma~\ref{lem:incoherence_vt}]
Using \eqref{eq:compvtupdate} and using the fact that $B, C$ are diagonal matrices:
$$\widehat{v}^{t+1}_j=\so\ip{\ut}{\uo}\vo_j-\frac{\so}{B_{jj}} \left(\ip{\ut}{\uo}B_{jj}-C_{jj}\right)\vo_j.$$
We bound the largest magnitude of elements in $\widehat{v}^{t+1}$ as follows. For every $j\in[n]$, we have:
\begin{align*}
  \left|\widehat{v}^{t+1}_j\right| &\leq \left|\so\ip{\ut}{\uo}\vo_j\right| + \left|\frac{\so}{B_{jj}} \left(\ip{\ut}{\uo}B_{jj}-C_{jj}\right)\vo_j\right| \\
	&\stackrel{(\zeta_1)}{\leq} \so \ip{\ut}{\uo} \frac{\mu}{\sqrt{n}} + \frac{\so}{1-\delta_2} \left(\ip{\ut}{\uo} \left(1+\delta_2\right) + \left(\ip{\ut}{\uo}+\delta_2\right)\right) \frac{\mu}{\sqrt{n}} \\
	&\leq \frac{\frac{3\so(1+\delta_2)\mu}{1-\delta_2}}{\sqrt{n}} \leq \frac{\so\mu_1}{2\sqrt{n}},
\end{align*}
where $(\zeta_1)$ follows from the fact that $1-\delta_2 \leq B_{jj} \leq 1+\delta_2$ and $\left|C_{jj}\right| \leq \left(\left|\ip{\ut}{\uo}\right|+\delta_2\right)$ (please refer Lemma \ref{lem:errb_comp2}).

Also, from \eqref{eq:ipvtvo_comp} we see that:
\begin{align*}
  \twonorm{\widehat{v}^{t+1}} \geq \ip{\widehat{v}^{t+1}}{\vo} &\geq \so\ip{\ut}{\uo}-2\so\delta_2\sqrt{1-\ip{\ut}{\uo}^2} \\
	&\geq \so\ip{u^0}{\uo}-2\so\delta_2\sqrt{1-\ip{u^0}{\uo}^2} \\
	&\stackrel{(\zeta_1)}{\geq} \frac{\so}{2},
\end{align*}
where $(\zeta_1)$ follows from the fact that $\dist\left(u^0,u^*\right) \leq \frac{3}{50}$ (please refer Lemma \ref{lemma:svd-clipping-jointguarantee}).
Using the above two inequalities, we obtain:
\begin{align*}
  \infnorm{v^{t+1}} = \frac{\infnorm{\widehat{v}^{t+1}}}{\twonorm{\widehat{v}^{t+1}}} \leq \frac{\left(\frac{\so\mu_1}{2\sqrt{n}}\right)}{\left(\frac{\so}{2}\right)}
	= \frac{\mu_1}{\sqrt{n}}.
\end{align*}
This finishes the proof.
\end{proof}

\begin{lemma}[Modified version of Lemma 6.1 of \cite{KeshavanOM2009}]
\label{lem:komlem}
Let $\Omega$ be a set of indices sampled uniformly at random from $[m]\times[n]$ with each element of $[m]\times[n]$ sampled
independently with probability $p\geq \frac{C\log n}{m}$. Then, w.p. at least $1-\frac{1}{n^3}$, $\forall x \in \mathbb{R}^m, y \in \rn$ s.t. $\sum_i x_i=0$, we have: 
$\sum_{ij\in \Omega}x_iy_j\leq C\sqrt{\sqrt{mn}p}\|x\|_2\|y\|_2,$ where $C>0$ is a global constant. 
\end{lemma}
\subsection{General Rank-$k$ Matrix Completion: Proofs}\label{app:compk}
\begin{proof}[Proof of Lemma~\ref{lemma:completion-incoherence-Vt}]
  From the decoupled update equation, \eqref{eqn:sampling-updateeqn-decoupled}, we obtain:
\begin{align*}
    (\Vt)^{(j)} = (R^{(t+1)})^{-1}(D^j-(B^j)^{-1}(B^jD^j-C^j)) \Sigma^*(\Vo)^{(j)},\;\;\;\;  \;\; 1\leq j\leq n.
\end{align*}
We bound the two norm of the $(\Vt)^{(j)}$ as follows:
\begin{align*}
  \twonorm{(\Vt)^{(j)}}
	&\leq \frac{\sigma_1\twonorm{(\Vo)^{(j)}}}{\sigma_{\textrm{min}}\left(R^{(t+1)}\right)} \left( \twonorm{D^j} + \twonorm{(B^j)^{-1}(B^jD^j-C^j)} \right) \\
	&\leq \frac{\sigma_1\twonorm{(\Vo)^{(j)}}}{\sigma_{\textrm{min}}\left(R^{(t+1)}\right)} \left( \twonorm{D^j} + \frac{\twonorm{B^jD^j}+\twonorm{C^j}}{\sigma_{\textrm{min}}\left(B^j\right)} \right) \\
	  &\stackrel{(\zeta_1)}{\leq} \frac{\sigma_1 \frac{\mu \sqrt{k}}{\sqrt{n}}}{\sigma_k^*\sqrt{1-\dist^2\left(U^{(t)},U^*\right)}-\frac{\sigma_1^* \delta_{2k}k\dist(U^{(t)},U^*)}{1-\delta_{2k}}} \left( 1 + \frac{(1+\delta_{2k})+ (1+\delta_{2k}) }{1-\delta_{2k}} \right) \\
	  &\leq \frac{4\sigma_1 \frac{\mu \sqrt{k}}{\sqrt{n}}}{\sigma_k^*\sqrt{1-\dist^2\left(U^{(0)},U^*\right)}-\frac{\sigma_1^* \delta_{2k}k\dist(U^{(0)},U^*)}{1-\delta_{2k}}}
	  \leq \frac{\left(\frac{16\sigma_1^*\mu}{\sigma_k^*}\right)\sqrt{k}}{\sqrt{n}},
\end{align*}
where we used the following inequalities in $(\zeta_1)$:
\begin{align}
  \twonorm{(\Vo)^j} &\leq \frac{\mu \sqrt{k}}{\sqrt{n}}, \label{eqn:Vstar-incoherence-int} \\
  \sigma_{\textrm{min}}\left(R^{(t+1)}\right) &\geq 
	\sigma^*_k\sqrt{1-\dist^2\left(U^{(t)},U^*\right)}-\sigma^*_1 \delta_{2k}k\dist(U^{(t)},U^*),
		\label{eqn:Rsigmamin-int} \\
  \sigma_{\textrm{min}}\left(B^j\right) \geq 1-\delta_{2k} &\mbox{ and }
	\sigma_{\textrm{max}}\left(B^j\right) \leq 1+\delta_{2k}, \label{eqn:Bjeigenvaluebound-int} \\
  \sigma_{\textrm{max}}\left(C^j\right) &\leq 1+\delta_{2k} \mbox{ and } \label{eqn:Cjeigenvaluebound-int} \\
  \sigma_{\textrm{max}}\left(D^j\right) &\leq 1, \label{eqn:Djeigenvaluebound-int}
\end{align}
where
\eqref{eqn:Vstar-incoherence-int} follows from the incoherence of $\Vo$,
\eqref{eqn:Rsigmamin-int} follows from from an analysis similar to the proof of Lemma \ref{lemma:sensing-minsingval-Rt},
\eqref{eqn:Bjeigenvaluebound-int} follows from (the proof of) Lemma \ref{lemma:sampling-minsingval-Bl},
\eqref{eqn:Cjeigenvaluebound-int} follows from Lemma \ref{lemma:sampling-minsingval-Cjl} and finally
\eqref{eqn:Djeigenvaluebound-int} follows from the fact that $D^j = \left(U^t\right)^{\dag}U^*$ with
$U^t$ and $U^*$ being orthonormal column matrices.
\end{proof}
\begin{proof}[Proof of Lemma~\ref{lemma:Fbound_comp}]
  Note that,
\begin{align}
  \twonorm{F(\Sigma^*)^{-1}} \leq \frob{F(\Sigma^*)^{-1}} &= \twonorm{B^{-1} \left(BD-C\right)v^*} \nonumber\\
	&\leq \twonorm{B^{-1}}\twonorm{(BD-C)v^*} \nonumber\\
	&\leq \frac{\delta_{2k}}{1-\delta_{2k}}\dist(\Ut, \Uo),   \label{eqn:sampling-twonorm-F-1}
\end{align}
where the last inequality follows using Lemma~\ref{lemma:sampling-minsingval-Bl} and Lemma~\ref{lemma:sampling-minsingval-Cl}. 
\end{proof}
We now bound $\|B^{-1}\|_2$ and $\|C^j\|_2$, which is required by our bound for $F$ as well as for our incoherence proof. 
\begin{lemma}\label{lemma:sampling-minsingval-Bl}
Let $M, \Omega, p,$ and $\Ut$ be as defined in Theorem~\ref{thm:sampling-altmin-main} and Lemma~\ref{lemma:Fbound_comp}. 
Then, w.p. at least $1-\frac{1}{n^3}$: 
\begin{align}\label{eqn:sampling-minsingval-Bl}
\|B^{-1}\|_2\leq \frac{1}{1-\delta_{2k}}.
\end{align}
\end{lemma}
\begin{proof}[Proof of Lemma~\ref{lemma:sampling-minsingval-Bl}]
We have:
$$\|B^{-1}\|_2=\frac{1}{\sigma_{min}(B)}=\frac{1}{\min_{x, \|x\|=1}x^\dag B x},$$
where $x\in \mathbb{R}^{nk}$. Let $x=vec(X)$, i.e.,  $x_p$ is the $p$-th column of $X$ and $x^j$ is the $j$-th row of $X$. Now, $\forall x$,
$$x^\dag B x=\sum_{j} (x^j)^\dag B^j (x^j)\geq min_j \sigma_{min}(B^j).$$
Lemma would follow using the bound on $\sigma_{min}(B^j), \forall j$ that we show below. 

\noindent {\bf Lower bound on $\sigma_{min}(B^j)$}: Consider any $w\in \R^k$ such that $\twonorm{w}=1$.
We have:
\begin{align*}
  Z=w^\dag  B^j w = \frac{1}{p}\sum_{i:(i,j)\in \Omega} \ip{w}{(\Ut)^{(i)}}^2= \frac{1}{p}\sum_{i}\delta_{ij}\ip{w}{(\Ut)^{(i)}}^2.
\end{align*}
Note that, $\E[Z]=w^\dag UU^\dag w=w^\dag w=1$ and $\E[Z^2]=\frac{1}{p}\sum_{i}\ip{w}{(\Ut)^{(i)}}^4\leq \frac{\mu_1^2k}{mp}\sum_{i}\ip{w}{(\Ut)^{(i)}}^2=\frac{\mu_1^2k}{mp},$ where the second last inequality follows using incoherence of $\Ut$. Similarly, $\max_{i}|\ip{w}{(\Ut)^{(i)}}^2|\leq \frac{\mu_1^2k}{mp}$. Hence, using Bernstein's inequality: 
$$\Pr(|Z-\E[Z]|\geq \delta_{2k})\leq \exp(-\frac{\delta_{2k}^2/2}{1+\delta_{2k}/3}\frac{mp}{\mu_1^2k}).$$
That is, by using $p$ as in the statement of the lemma with the above equation and using union bound, we get (w.p. $>1-1/n^3$): $\forall w,j\ \ w^\dag B^j w \geq 1-\delta_{2k}$. That is, $\forall j, \sigma_{min}(B^j)\geq (1-\delta_{2k})$. 
\end{proof}
\begin{lemma}\label{lemma:sampling-minsingval-Cjl}
Let $M, \Omega, p,$ and $\Ut$ be as defined in Theorem~\ref{thm:sampling-altmin-main} and Lemma~\ref{lemma:Fbound_comp}. Also, let $C^j \in \mathbb{R}^{k\times k}$ be defined as: $C^j=\frac{1}{p}\sum_{i:(i,j)\in \Omega} (\Ut)^{(i)}{(\Uo)^{(i)}}^\dag$. Then, w.p. at least $1-\frac{1}{n^3}$: 
\begin{align}\label{eqn:sampling-minsingval-Cjl}
\|C^j\|_2\leq 1+\delta_{2k}, \forall j
\end{align}
\end{lemma}
\begin{proof}[Proof of Lemma~\ref{lemma:sampling-minsingval-Cjl}]
Let $x\in \mathbb{R}^{k}$ and $y\in \mathbb{R}^k$ be two arbitrary unit vectors. Then, 
$$x^TC^jy=\frac{1}{p}\sum_{i:(i,j)\in \Omega}(x^\dag(\Ut)^{(i)})(y^\dag(\Uo)^{(i)}).$$
That is, $Z=x^TC^jy=\frac{1}{p}\sum_{i}\delta_{ij}(x^\dag(\Ut)^{(i)})(y^\dag(\Uo)^{(i)})$. Note that, $\E[Z]=x^\dag (\Ut)^\dag\Uo y$, $\E[Z^2]=\frac{1}{p}\sum_{i}(x^\dag(\Ut)^{(i)})^2(y^\dag(\Uo)^{(i)})^2\leq \frac{\mu^2}{mp}x^\dag(\Ut)^\dag\Ut x=\frac{\mu^2k}{mp}$ and $\max_i |(x^\dag(\Ut)^{(i)})(y^\dag(\Uo)^{(i)})|\leq \frac{\mu_1^2k}{m}$. Lemma now follows using Bernstein's inequality and using bound for $p$ given in the lemma statement. 
\end{proof}
Finally, we provide a lemma to bound the second part of the error term ($F$).
\begin{lemma}\label{lemma:sampling-minsingval-Cl}
Let $M, \Omega, p,$ and $\Ut$ be as defined in Theorem~\ref{thm:sampling-altmin-main} and Lemma~\ref{lemma:Fbound_comp}.
Then, w.p. at least $1-\frac{1}{n^3}$: 
\begin{align}\label{eqn:sampling-minsingval-Cl}
\|(BD-C)v^*\|_2\leq  \delta_{2k} \dist(\Vt, \Vo),
\end{align}
where $v^*=vec(\Vo)$, i.e. $v^*=\left[\begin{matrix}\Vo_1\\ \vdots \\ \Vo_k\end{matrix}\right]$. 
\end{lemma}
\begin{proof}[Proof of Lemma~\ref{lemma:sampling-minsingval-Cl}]
Let $X\in \mathbb{R}^{n\times k}$ and let $x=vec(X)\in \mathbb{R}^{nk}$ s.t. $\|x\|_2=1$. Also, let $x_p$ be the $p$-th column of $X$ and $x^j$ be the $j$-th column of $X$. 

Let $u^i=(\Ut)^{(i)}$ and $u^{*(i)}=(\Uo)^{(i)}$. Also, let $H^j=(B^jD-C^j)$, i.e., 
$$H^j=\frac{1}{p}\sum_{i:(i,j)\in \Omega}u^i(u^i)^\dag (\Ut)^\dag \Uo - u^i (u^{*(i)})^\dag =\frac{1}{p}\sum_{i:(i,j)\in \Omega} H^j_i,$$ where $H^j_i\in \mathbb{R}^{k\times k}$. Note that, 
\begin{equation}\label{eq:ch_k}\sum_{i}H^j_i=(\Ut)^\dag \Ut (\Ut)^\dag\Uo - (\Ut)^\dag \Uo = 0.\end{equation}
Now, $x^\dag(BD-C)v^*= \sum_j(x^j)^\dag(B^jD-C^j)(\Vo)^{(j)}=\frac{1}{p}\sum_{pq}\sum_{(i,j)\in \Omega}x^j_p\Vo_{jq}(H^j_i)_{pq}$. Also, using \eqref{eq:ch_k}, $\forall (p,q)$: $$\sum_{i}(H^j_i)_{pq}=0.$$
Hence, applying Lemma~\ref{lem:komlem}, we get w.p. at least $1-\frac{1}{n^3}$:
\begin{equation}\label{eq:bcdbound_compk}x^\dag (BD-C)v^*=\sum_j(x^j)^\dag(B^jD-C^j)(\Vo)^{(j)}\leq \frac{1}{p}\sum_{pq}\sqrt{\sum_j(x^j_p)^2(\Vo_{jq})^2}\sqrt{\sum_i(H^j_i)_{pq}^2}.\end{equation}
Also, 
\begin{align}
\sum_i(H^j_i)_{pq}^2&=\sum_i(u^i_p)^2((u^i)^\dag (\Ut)^\dag\Uo_q-\Uo_{iq})^2\leq \max_i(u^i_p)^2\sum_i((u^i)^\dag (\Ut)^\dag\Uo_q-\Uo_{iq})^2\nonumber\\&=\max_i(u^i_p)^2 (1-\|\Ut\Uo_q\|_2^2) \leq \frac{\mu_1^2k}{m}\dist(\Ut, \Uo)^2.\label{eq:hbound_compk}\end{align}
Using \eqref{eq:bcdbound_compk}, \eqref{eq:hbound_compk} and incoherence of $\Vo$, we get (w.p. $1-1/n^3$), $\forall x$: 
$$x^\dag(BD-C)v^*\leq \sum_{pq}\frac{\mu_1^2k}{mp}\dist(\Ut, \Uo)\|x_p\|_2\leq \delta_{2k}\dist(\Ut, \Uo),$$
where we used the fact that $\sum_p \twonorm{x_p} \leq \sqrt{k}\twonorm{x} = \sqrt{k}$ in the last step.
Lemma now follows by observing $\max_{x, \|x\|=1} x^\dag(BD-C)v^*=\|(BD-C)v^*\|_2$. 
\end{proof}



\end{document}